\documentclass[hidelinks,onefignum,onetabnum]{siamart250211}

\usepackage{lipsum}
\usepackage{amsfonts}
\usepackage{amssymb}
\usepackage{graphicx}
\usepackage{epstopdf}
\usepackage{algorithmic}
\usepackage{subfigure}

\ifpdf
  \DeclareGraphicsExtensions{.eps,.pdf,.png,.jpg}
\else
  \DeclareGraphicsExtensions{.eps}
\fi

\newif\ifcleanversion
\cleanversiontrue

\usepackage[normalem]{ulem} 

\ifcleanversion
    \newcommand{\revadd}[1]{#1}
    \newcommand{\revdel}[1]{}
\else
    \newcommand{\revadd}[1]{\textcolor{red}{#1}}
    \newcommand{\revdel}[1]{\textcolor{red}{\sout{#1}}}
\fi

\newsiamremark{remark}{Remark}
\newsiamremark{hypothesis}{Hypothesis}
\crefname{hypothesis}{Hypothesis}{Hypotheses}
\newsiamthm{assumption}{Assumption}
\newsiamthm{claim}{Claim}
\newsiamremark{fact}{Fact}
\crefname{fact}{Fact}{Facts}

\headers{Correcting AD Neural-ODE Training}{Y. Xu, S. Chen, and Q. Li}

\title{Correcting Auto-Differentiation in Neural-ODE Training\thanks{Submitted to the editors on June 9th, 2025.
\funding{Q.L. acknowledges support from Vilas Early Career award and ONR-N00014-21-1-214. The research of Y.X., S.C., and Q.L. is supported in part by NSF-CAREER-1750488 and Office of the Vice Chancellor for Research and Graduate Education at the University of Wisconsin Madison with funding from the Wisconsin Alumni Research Foundation. The work of Y.X., S.C. and Q.L. is supported in part by NSF via grant DMS-2023239. }}}

\author{Yewei Xu\thanks{Department of Mathematics, University of Wisconsin-Madison, Madison WI 53706 (\email{xu464@wisc.edu}).}
\and Shi Chen\thanks{Department of Mathematics, Massachusetts Institute of Technology, Cambridge MA 02139
  (\email{schen636@mit.edu}, \url{https://simonchenthu.github.io/}).} 
\and Qin Li\thanks{Department of Mathematics, University of Wisconsin-Madison, Madison WI 53706
  (\email{qinli@math.wisc.edu}, \url{https://people.math.wisc.edu/\string~qinli/}).}
}

\usepackage{amsopn}

\newcommand{\NN}{\mathsf{NN}}
\newcommand{\rp}{\mathsf{p}}
\newcommand{\sfq}{\mathsf{q}}
\newcommand{\rT}{\mathsf{T}}
\newcommand{\rz}{\mathsf{z}}

\newcommand{\calE}{\mathcal{E}}
\newcommand{\new}{\mathrm{new}}
\newcommand{\old}{\mathrm{old}}
\newcommand{\rxi}{\mathsf{\xi}}

\newcommand{\rtheta}{\mathsf{\theta}}

\newcommand{\rd}{\mathrm{d}}
\newcommand{\calO}{\mathcal{O}}

\ifpdf
\hypersetup{
  pdftitle={Correcting Auto-Differentiation in Neural-ODE Training},
  pdfauthor={Y. Xu, S. Chen, and Q. Li}
}
\fi

\makeatletter
\def\@suppseccntformat#1{\csname the#1\endcsname\quad}

\newcommand{\supplement}{%
  \clearpage
  \begin{center}
    \LARGE\bfseries Supplementary Material
  \end{center}
  \vspace{1em}

  \setcounter{section}{0}
  \setcounter{subsection}{0}
  \setcounter{subsubsection}{0}
  \setcounter{equation}{0}
  \setcounter{figure}{0}
  \setcounter{table}{0}
  \@ifundefined{c@theorem}{}{\setcounter{theorem}{0}}
  \@ifundefined{c@lemma}{}{\setcounter{lemma}{0}}
  \@ifundefined{c@proposition}{}{\setcounter{proposition}{0}}
  \@ifundefined{c@corollary}{}{\setcounter{corollary}{0}}
  \@ifundefined{c@definition}{}{\setcounter{definition}{0}}
  \@ifundefined{c@remark}{}{\setcounter{remark}{0}}
  \@ifundefined{c@assumption}{}{\setcounter{assumption}{0}}
  \@ifundefined{c@claim}{}{\setcounter{claim}{0}}
  \@ifundefined{c@fact}{}{\setcounter{fact}{0}}
  \@ifundefined{c@hypothesis}{}{\setcounter{hypothesis}{0}}

  \renewcommand\thesection{SM\arabic{section}}
  \renewcommand\thesubsection{SM\arabic{section}.\arabic{subsection}}
  \renewcommand\thesubsubsection{SM\arabic{section}.\arabic{subsection}.\arabic{subsubsection}}

  \let\@seccntformat\@suppseccntformat

  \renewcommand{\section}{\@startsection{section}{1}{\z@}%
    {-3.5ex \@plus -1ex \@minus -.2ex}%
    {2.3ex \@plus.2ex}%
    {\normalfont\normalsize\bfseries}}

  \renewcommand{\subsection}{\@startsection{subsection}{2}{\z@}%
    {-3.25ex \@plus -1ex \@minus -.2ex}%
    {1.5ex \@plus .2ex}%
    {\normalfont\normalsize\bfseries}}

  \renewcommand{\subsubsection}{\@startsection{subsubsection}{3}{\z@}%
    {-3.25ex \@plus -1ex \@minus -.2ex}%
    {1.5ex \@plus .2ex}%
    {\normalfont\normalsize\itshape}}

  \renewcommand{\theequation}{SM(\arabic{equation})}
  \renewcommand{\thefigure}{SM\arabic{figure}}
  \renewcommand{\thetable}{SM\arabic{table}}
  \renewcommand{\thetheorem}{SM\arabic{section}.\arabic{theorem}}
  \renewcommand{\thelemma}{SM\arabic{section}.\arabic{lemma}}
  \renewcommand{\theproposition}{SM\arabic{section}.\arabic{proposition}}
  \renewcommand{\thecorollary}{SM\arabic{section}.\arabic{corollary}}
  \renewcommand{\thedefinition}{SM\arabic{section}.\arabic{definition}}
  \renewcommand{\theremark}{SM\arabic{section}.\arabic{remark}}
  \renewcommand{\theassumption}{SM\arabic{section}.\arabic{assumption}}
  \renewcommand{\theclaim}{SM\arabic{section}.\arabic{claim}}
  \renewcommand{\thefact}{SM\arabic{section}.\arabic{fact}}
  \renewcommand{\thehypothesis}{SM\arabic{section}.\arabic{hypothesis}}
}
\makeatother

\begin{document}

\maketitle
\begin{abstract}
Does the use of auto-differentiation yield reasonable updates for deep neural networks (DNNs)? Specifically, when DNNs are designed to adhere to neural ODE architectures, can we trust the gradients provided by auto-differentiation? Through mathematical analysis and numerical evidence, we demonstrate that when neural networks employ high-order methods, such as Linear Multistep Methods (LMM) or Explicit Runge-Kutta Methods (ERK), to approximate the underlying ODE flows, brute-force auto-differentiation often introduces artificial oscillations in the gradients that prevent convergence. In the case of Leapfrog and 2-stage ERK, we propose simple post-processing techniques that effectively \revdel{eliminates} \revadd{eliminate} these oscillations, correct the gradient computation and thus returns the accurate updates.
\end{abstract}



\section{Introduction}

\begin{figure*}[ht]
\begin{center}
    \subfigure[]{\label{fig:failure_auto_1_leapfrog}
    \includegraphics[width=0.475\textwidth]{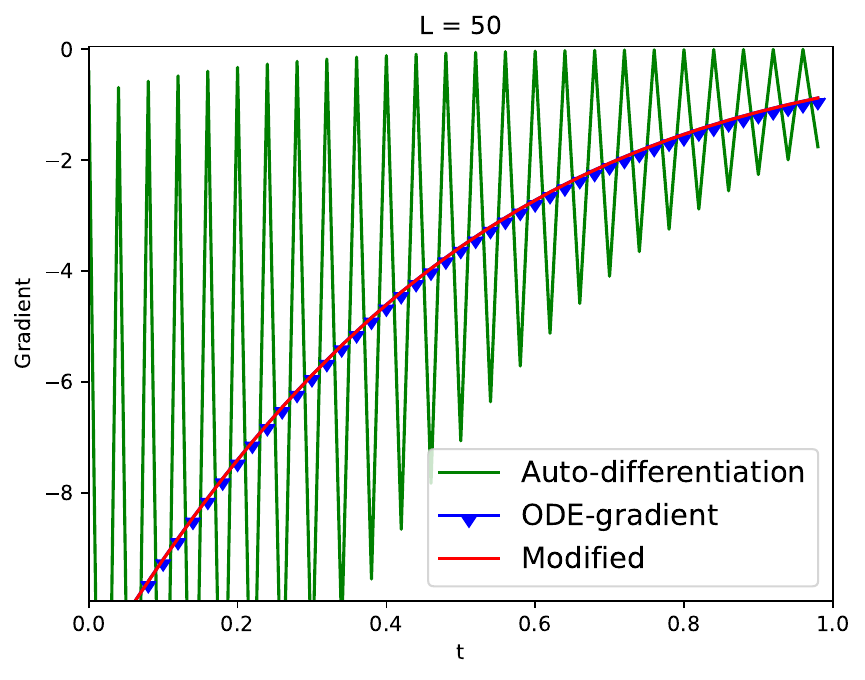}
    }
    \subfigure[]{\label{fig:failure_auto_1_midpoint}
    \includegraphics[width=0.475\textwidth]{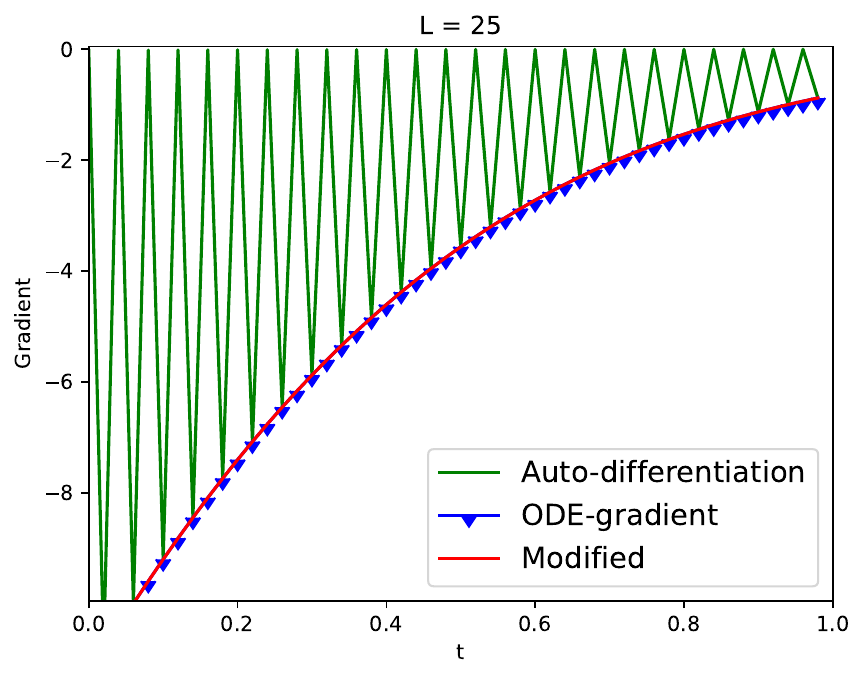}
    } 
    \subfigure[]{\label{fig:failure_auto_1_ralston}
    \includegraphics[width=0.475\textwidth]{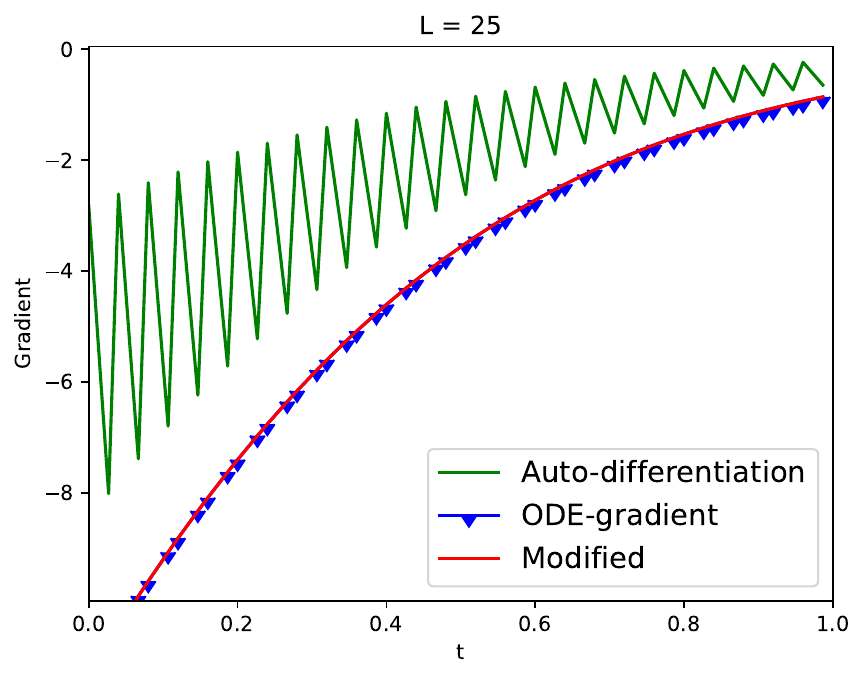}
    }
    \subfigure[]{\label{fig:failure_auto_1_nystrom}
    \includegraphics[width=0.475\textwidth]{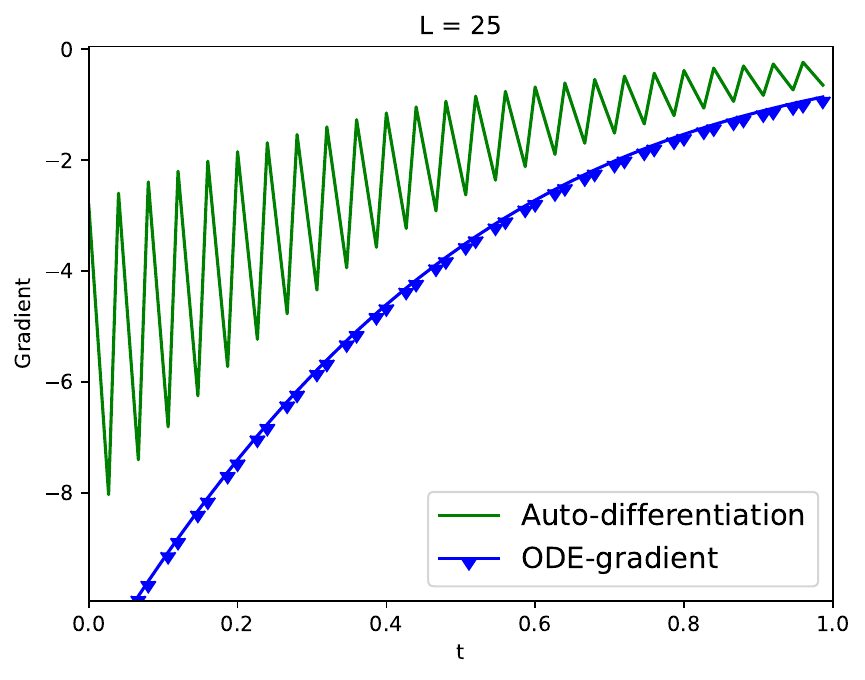}
    }
\caption{The four panels show gradients computed for four different discretization schemes -- Leapfrog (a), Midpoint (b), Ralston (c) and \revdel{Nystrom} \revadd{Nystr\"{o}m} (d). In all four cases, the gradients computed using auto-differentiation (green \revadd{oscillatory solid lines}) generate artificial severe oscillations and deviate from the true gradient of the underlying ODE (blue \revadd{curves with triangular markers}). In this paper, we propose a correction method (red \revadd{curves overlapping with the blue curves marked by triangles} in all but the last panel) that, through a very simple algebraic manipulation of auto-differentiation, returns gradients that agree with the underlying ODE updates.}
\label{fig:failure_auto}
\end{center}
\end{figure*}

Can we trust auto-differentiation for computing gradients in the training of neural networks (NNs)? When considering NNs that represent neural ODEs, this general question simplifies to evaluating whether the auto-differentiation returns a good update to the ODE parameters.

Neural ODE~\cite{ChRuBeDu:2018neural} is a class of prominent deep neural network models and has been investigated widely for various tasks~\cite{RuChDu:2019latent,DuDoTe:2019augmented,GhKeBi:2019anode,PaNaReMoLa:2021normalizing,RaMaMaWaZuSuSkRaEd:2020universal,BrNoKo:2020machine,Choi2023,NEURIPS2020_4db73860,scholz2023latent,wang2021forecasting,Zhou_Li_2021,Chen2022,10.1145/3534678.3539245,Chen2020Symplectic,pmlr-v161-cobb21a,10.5555/3504035.3504378,8909789}. 
This architecture exploits some ODE structures and views the NN layers as the discrete-in-time representation of the ODE flows. 
Since its introduction, various perspectives have been taken to examine the performance of auto-differentiation on ODE-nets. 
As a black-box tool, auto-differentiation is easy-to-use and usually returns a reliable computation of gradients. 
Nevertheless, already in~\cite{ChRuBeDu:2018neural}, the authors argued that auto-differentiation requires unnecessary computation, but instead, they proposed a cheaper direct simulation of an adjoint solver for assembling the gradient.

We take on a different perspective by shifting our focus from computational efficiency to accuracy. 
Suppose NN is built as a discretization of forward Euler type for the underlying ODE~\footnote{Such models are also known as Deep Residual Networks (Deep ResNets)~\cite{E17}.}, it was shown in earlier results~\cite{Lu19:Meanfield} that auto-differentiation seems to return \revdel{reasonable} gradients \revadd{that are numerically stable and often lead to accurate updates}. 
For high-order discretized representation of the underlying flow, we find the story is vastly different. 
Indeed, as demonstrated in Figure~\ref{fig:failure_auto}, four very simple high-order discretization were employed to translate an underlying ODE flow to an NN architecture. 
They are Linear Multistep Methods (LMM), Midpoint, Ralston and \revdel{Nystrom} \revadd{Nystr\"{o}m} respectively, with the latter three fall in the category of Explicit Runge-Kutta Methods (ERK). 
For all of them, direct application of auto-differentiation returns gradients that suffer from severe artificial oscillations, and deviate from the true gradients computed from the underlying ODEs. 
Moreover, among the three ERK methods, midpoint method captures the correct gradient information at half-integer points, and returns zeros for integer points. 
This means only half-integer points get updated in the training process. 
In comparison, Ralston and \revdel{Nystrom} \revadd{Nystr\"{o}m} give gradients that are completely wrong.

This strange behavior prompts the following questions:

\medskip

\begin{itemize}
\item \emph{What is the mechanism that drives the artificial oscillations?}
\item \emph{What makes some ERK methods accurate only at specific points?}
\item \emph{Can we remove these oscillations?}
\end{itemize}
  
\medskip

In this paper, we answer all these questions. We will examine the problem through the disparity between the ``discretize-then-optimize" (DTO) and ``optimize-then-discretize" (OTD)~\cite{BeHeLoEsWe19, OnRu20:DO, LiWa19:NonComm, HiRo12:DisOpCo, ReDoWa10:OpSoPDE, LlJe24:Di}, and provide a mathematical explanation to these oscillations. 
For solving differential-equation-constrained optimization, it has been well-known that the order of operations matters in the updates. 
In the DTO setting, discretization is conducted first to translate an ODE to an algebraic problem, upon which optimization is performed. 
The OTD approach swaps the order of these operations, and performs optimization directly on the equation level in the continuous framework, before discretizing the optimization solver for an algebraic update. 
Auto-differentiation is a procedure that completely operates on an algebraic system, and thus naturally falls into the category of DTO.

In theory, DTO and OTD both aim at providing gradients and updates for the same system, so ideally the order should not matter. 
This is not true in general. 
Whether the ``discretize'' step and ``optimize'' step commute is highly dependent on the specific discretization scheme deployed. 
\revadd{Related inconsistencies have also appeared in recent machine-learning contexts.
For instance, physics-informed neural networks (PINNs) often optimize a continuous PDE residual while using discrete solvers, and symplectic or structure-preserving neural ODEs exhibit similar issues when differentiating through continuous-time adjoints~\cite{MMY21,LZWLWLSL25,ZLM25,RKGS25}.
These mismatches are typically addressed by enforcing differentiation at the discrete level, either through discrete-loss formulations or adjoint schemes consistent with the solver.}

\revadd{These discrepancies have also been systematically analyzed in the context of neural ODEs.
The original neural ODE framework adopted the OTD viewpoint via continuous adjoints~\cite{ChRuBeDu:2018neural}, 
while subsequent works explored discrete or solver-consistent adjoints to improve numerical accuracy and efficiency~\cite{GhKeBi:2019anode, OnRu20:DO, FJNO20, KJDMR21}.
Although these studies revealed that the two formulations may diverge depending on the discretization and stiffness of the dynamics, a precise characterization of the mechanism behind such discrepancies has remained limited.}

\revdel{Indeed, to} \revadd{To} assemble a gradient \revadd{in our ODE-Net setting}, a solution to a forward equation and a solution to an adjoint are both required, with the forward equation solving the original equation forward in time, and the adjoint equation propagating the mismatch back to the initial time. 
\revdel{These two equations need to be discretized in a manner that is compatible with each other, so to ensure solving forward and adjoint in discrete level separately, agrees with discretizing the entire forward and adjoint combo as one unit.} \revadd{These two equations must be discretized in a mutually consistent manner, so that solving the forward and adjoint equations separately at the discrete level yields the same result as discretizing the coupled forward–adjoint system as a whole.}

The failure of auto-differentiation in the setting of leapfrog is exactly rooted in such incompatibility. 
The ``leap" step generates a two-step dependence for updating states, and, in nature, is a procedure that averages out the impact of history. 
Yet, the objective function only sees the value of the very final stage of state and is not averaging in nature. 
Such incompatibility eventually leads to the oscillatory pattern shown in Figure~\ref{fig:failure_auto_1_leapfrog}. 
This does not mean, however, that auto-differentiation is completely irrelevant. Gradient information is computed, but incorrectly assembled. 
This understanding above brings about a potential \emph{post-processing} strategy. 
By building the ``leap" structure back in the gradient computation through averaging, it is still possible to integrate the ``averaging" feature of leapfrog, and to extract out the true gradients. 
This is indeed what we propose. 
By deploying a very simple post-processing linear operation, we obtain a much more accurate gradient, with the oscillation completely eliminated, see red curves in Figure~\ref{fig:failure_auto_1_leapfrog}.
This post-processing technique brings auto-differentiation back to relevance. A similar strategy is designed for ERK-2 stage as well. 

\revadd{We should stress that the failure of auto-differentiation has been observed, and further categorized, for example, in~\cite{HuMeMoChHoHa23}, in which the authors list three different phenomena where auto-differentiation results mismatch the real differentiation. When deployed specifically for neural ODE and ODE-nets training, the failure of auto-differentiation falls in their ``Pitfall II: the chosen abstraction has unexpected derivatives.'' Our work specifically identifies how this abstraction mismatch materializes in Leapfrog and two-stage ERK ODE-nets, derives the precise mechanism behind the spurious oscillations, and proposes a minimal post-processing correction that restores the true gradient. 
This concrete characterization and remedy extend the abstract discussion in~\cite{HuMeMoChHoHa23} to a practical, provably correct setting for ODE-based neural networks.} \revdel{offer a straightforward approach to address these challenges for ODE-nets.}

The paper is organized as follows. In Section~\ref{sec:prelim}, we review the concepts of neural ODEs and ODE-nets, and present the computation on the adjoint and the gradient of ODE-nets (Proposition~\ref{prop:grad_dis} and~\ref{prop:p-adjoint-ERK}). 
These explicit calculations allow us to provide an explanation for the oscillations in auto-differentiation, as presented in Section~\ref{sec:oscillation_autodiff}. 
In Section~\ref{sec:correction}, we propose, for Leapfrog and 2-stage ERK methods respectively, a post-processing technique that corrects the auto-differentiation gradient. 
We show that this correction attains second order accuracy; see Theorem~\ref{theorem:LMMmain} and~\ref{theorem:ERK-correction}.
Finally, in Section~\ref{sec:numerical_examples}, we close with several numerical examples, verifying the second order accuracy of our proposed correction, and demonstrating that our correction improves the performance of ODE-nets.

\section{Preliminaries}\label{sec:prelim}

\subsection{Neural ODEs and ODE-nets}\label{sec:prelim-neural-ODE-def}

Throughout the paper, we term the formulation in the continuous setting neural ODE, and ODE-nets are the associated NN structures that can be viewed as discretization of an neural ODE.

In a general form, neural ODE is written as:
\begin{equation}\label{eqn:ODE}
\frac{\rd z(t;x)}{\rd t} = f(z(t;x),\theta(t))\,,\quad\text{with}\quad z(0;x)=x\,.
\end{equation}
\revadd{Here $z(t;x)\in\mathbb{R}^d$ denotes the hidden state of dimension~$d$, and $\theta(t)\in\mathbb{R}^n$ is a time-dependent external weight function of dimension~$n$.}
It evolves initial value $z(0;x) =x$ using the activation function $f:\mathbb{R}^d\times\mathbb{R}^n\to\mathbb{R}^d$ up to time $t=1$. 
\revdel{The activation function $f$ is parameterized by an external weight function $\theta(t)\in\mathbb{R}^n$ and the final state of the ODE is $z(1;x)$.}
Often a measuring function $g$ is added at the very end that converts the final state to the output, so the whole neural-ODE structure maps the input initial data $x$ to the output $y = g(z(t=1;x))$.

ODE-nets translate this neural ODE to an NN structure through different versions of discretization. 
The most common one is obtained through the forward Euler method. 
Denote $L$ the depth of the NN, then step-size is $h=\frac{1}{L}$, and for $l=0,1,\cdots, L-1$:
\begin{equation}\label{eqn:forwardEuler}
\rz_{l+1} = \rz_{l} +hf(\rz_l, \theta_{l})\,,\quad\text{with}\quad \rz_0=x\,.
\end{equation}
All $\rz_l$ above have $x$-dependence through initial data. 
For conciseness, we omit this dependence in our notations. 
This formulation is obtained by approximating $\frac{\rd z(t_l;x)}{\rd t}$ using $\frac{\rz_{l+1} -\rz_{l}}{h}$, and it attains a first-order accuracy in the sense that, for the same initial data $x$, running~\eqref{eqn:forwardEuler} agrees with running~\eqref{eqn:ODE} with $\mathcal{O}(h)$ error:
\[
\text{if}\quad \theta_{l}=\theta(lh)\,,\quad\text{then}\quad \rz_L - z(t=1)=\mathcal{O}(h)\,.
\]

Higher-order methods are also available. 
One such example is the Leapfrog method, a classical example from the class of methods called the Linear Multistep Method~\cite{Is:2009first}. 
It has the following form. 
For $l=0, 1, \dots, L-2$, omitting the initial data $x$-dependence:
\begin{equation}\label{eqn:leapfrog}
\begin{cases}
\rz_{l+2} = \rz_{l} + 2hf(\rz_{l+1},\theta_{l+1})\,,\\
\rz_0=x\;,\; \rz_1=\rz_0+hf(\rz_0,\theta_0)\,.
\end{cases}
\end{equation}
The scheme leapfrogs from $l$ to $l+2$ using updates obtained at the $(l+1)$-th step. 
One crucial feature of the method is that, due to the \revdel{inavailability} \revadd{unavailability} of $z_{-1}$, the first step $z_1$ is obtained using a simple forward Euler update as in~\eqref{eqn:forwardEuler}. 
\revadd{Note that this initialization does not affect the overall second-order accuracy of the Leapfrog scheme, as confirmed in Lemma~\ref{lem:Delta_z}.}

ERK is another class of methods that provide high-order accuracy. 
Setting $\rz_0  = x$, a general formulation for an $\nu$-stage ERK method computes $\rz_{l+1}$ from $\rz_{l}$ using an $\nu$-stage updating formula. 
More specifically, the time domain $[t_l,t_{l+1}]$ is decomposed into multiple stages of $t_{l,i} = (l+c_i)h$ with $0=c_1 \leq c_2 \leq \cdots \leq c_\nu < 1$ standing for the time of intermediate stages. 
Denote $\rxi_{l,i}$ the solution at stage $i$ in $l$-th step, we can stack them up in a vector form $\rxi_l = [\rxi_{l,1}^\top \,, \rxi_{l,2}^\top \,, \ldots \,,\ \rxi_{l,\nu}^\top]^\top$ and similarly denote 
$f_l = [
f^\top(\rxi_{l,1}, \theta_{l, c_1}) \,, f^\top(\rxi_{l,2}, \theta_{l, c_2}) \,, \ldots \,, f^\top(\rxi_{l,\nu}, \theta_{l, c_\nu})]^\top$, then the general ERK presents:
\begin{equation}\label{eqn:ERK-def}
\begin{aligned}
& \rxi_l = \mathsf{e}_\nu \otimes \rz_l +h (\mathsf{A} \otimes \mathsf{I}_d)\cdot f_l\,, \\
& \text{and}\quad \rz_{l+1}=\rz_l+h (\mathsf{b}^\top \otimes \mathsf{I}_d) \cdot f_l\,,
\end{aligned}
\end{equation}
where $\mathsf{e}_\nu$ is a column vector of $\nu$ element with each element being $1$, $\mathsf{I}_d$ is the $d \times d
$ identity matrix, $\otimes$ is the Kronecker product notation, and $\mathsf{A}$ and $\mathsf{b}$ contain coefficients that are pre-defined in a Butcher's tableau:
\begin{equation}\label{eqn:coeff-mat-ERK-def}
\mathsf{A} = \begin{bmatrix}
0 & \\
a_{2,1} & 0 \\
a_{3,1} & a_{3,2} & 0 & \dots \\
\vdots \\
a_{\nu,1} & a_{\nu,2} & a_{\nu, 3} & \dots & a_{\nu,\nu-1} & 0 
\end{bmatrix}\,, \quad \mathsf{b} = \begin{bmatrix}
b_1 \\ b_2 \\ \vdots \\ b_\nu
\end{bmatrix} \, .
\end{equation}
These coefficients need to satisfy certain conditions~\cite{Is:2009first} to ensure high order of accuracy. Three classical examples of ERK are
\begin{itemize}
    \item{Midpoint method}
    \begin{equation}\label{eqn:midpoint}
    \begin{cases}
\rz_{l+\frac{1}{2}}=\rz_l+\frac{h}{2}f(\rz_l,\theta_l)\,,\\
\rz_{l+1} = \rz_{l} + hf(\rz_{l+\frac{1}{2}},\theta_{l,\frac{1}{2}})\,.
\end{cases}
\end{equation}
    \item{Ralston method}
    \begin{equation}\label{eqn:ralston}
    \begin{cases}
\rz_{l+\frac{2}{3}}=\rz_l+\frac{2h}{3}f(\rz_l,\theta_l)\,,\\
\rz_{l+1} = \rz_{l} + \frac{h}{4}f(\rz_{l},\theta_l) + \frac{3h}{4}f(\rz_{l+\frac{2}{3}}, \theta_{l,\frac{2}{3}})\,.
\end{cases}
\end{equation}
\item{\revdel{Nystrom} \revadd{Nystr\"{o}m} method} is a 3-stage ERK and uses the following set of coefficients
\begin{equation}\label{eqn:Nystrom-def}
\begin{aligned}
& a_{2,1} = a_{3,2} = c_{2} = c_{3} = \frac{2}{3}, a_{3,1} = c_1 = 0, \\
& b_1 = \frac{1}{4}, b_2 = b_3 = \frac{3}{8}\,.
\end{aligned}
\end{equation}
\end{itemize}

Among the methods mentioned above, Leapfrog, midpoint and Ralston are all second order accurate, meaning
\begin{equation}\label{eqn:2nd_order_z}
\text{if}\quad \theta_{l,c_j}=\theta((l+c_j)h)\quad\text{then}\quad\rz_L - z(t = 1) = \mathcal{O}(h^2)\quad\,.
\end{equation}
Similarly, \revdel{Nystrom} \revadd{Nystr\"{o}m} method can be shown to be third order accurate.

High order methods promote computational efficiency. Setting $\tau$ the error tolerance. To achieve this error tolerance, one needs $h\lesssim \tau$ if a first order accurate method is deployed, and this makes $L\approx \frac{1}{h} \propto\frac{1}{{\tau}}$ layers. In comparison, if a second order method is deployed, only $L\approx \frac{1}{\sqrt{\tau}}$ layers is necessary. Since $\frac{1}{\sqrt{\tau}}\ll\frac{1}{{\tau}}$ for small $\tau$, a second order method can achieve the same level of accuracy with a much smaller NN.
\revadd{Some recent works explicitly exploit this higher-order effect for numerical savings. For instance, Runge-Kutta convolutional networks (RKCNNs)~\cite{ZCF23} incorporate multi-stage Runge-Kutta updates within convolutional layers and empirically demonstrate an improved accuracy at comparable computational cost. Similar studies are complete on second-order neural ODEs (SONODEs) in~\cite{NBDSL20} where higher-order dynamical couplings enable better stability and convergence.}

\subsection{Training of neural ODEs and ODE-nets}

Neural ODEs and ODE-nets output results that are approximate to each other when the list of coefficients $\Theta=\{\theta_{l,c_j}\}$ agree with the continuous counterparts. In practice, during the training process, we use the labeled data to find $\theta(t)$ in~\eqref{eqn:ODE} or $\Theta=\{\theta_{l,c_j}\}$ in~\eqref{eqn:leapfrog} or~\eqref{eqn:ERK-def}. In this section, we discuss the training process. Throughout the section, we denote $(x,y) \in \mathbb{R}^d \times \mathbb{R}^{\revadd{m}}$ the training dataset \revadd{(here $m$ is the output dimension of the model)}, and $\mu$ the \revdel{distribution of the training data} \revadd{weight defining the weighted $\ell^2$ norm 
$\|v\|_\mu^2 := \langle v, v \rangle_\mu$, 
where $\langle \cdot , \cdot \rangle_\mu$ is the corresponding weighted inner product}.

To train neural ODE is to find $\theta(t)$ in~\eqref{eqn:ODE}. 
We may define the loss function as follows:
\begin{equation} \label{eqn:def.Etilde}
\begin{aligned}
& \calE(\theta) := \frac{1}{2}\|g(z^{\theta}(1;x))-y\|^2_\mu \, , \\
& \mbox{where $z^{\theta}(t;x)$ solves \eqref{eqn:ODE}.}
\end{aligned}
\end{equation}
Knowing that $\theta(t)$ is a function, $\calE$ maps a function of time to a number, and thus a functional. If gradient-descent is deployed, we have the updates for $\theta$:
\[
\theta^\new(t)=\theta^\old(t)-\eta\left.\frac{\delta\calE}{\delta\theta}\right|_{\theta^\old(t)}
\]
where $\left.\frac{\delta\calE}{\delta\theta}\right|_{\theta^\old(t)}$ is the Fr\'echet derivative of the functional $\calE$ on function $\theta(t)$ evaluated at $\theta^\old$, and $\eta$ is the step-size for training. Through the standard calculus-of-variation strategy, we can compute this derivative explicitly:
\begin{proposition}
Let the functional $\calE(\theta(t))$ \revdel{of $\theta(t)$} \revadd{be} defined \revadd{as} in~\eqref{eqn:def.Etilde}\revdel{, the}\revadd{. The} functional derivative is
\begin{equation}\label{eqn:gradient_functional}
\left.\frac{\delta \calE}{\delta \theta}\right|_{\theta(t)}(t) = p^\top(t) \partial_\theta f(z(t), \theta(t))\,,
\end{equation}
where $\partial_\rtheta f$ is the Jacobian of $f$ on $\theta$ and has the size of $\mathbb{R}^{d\times n}$, and $p\in\mathbb{R}^d$ solves the adjoint equation:
\begin{equation}\label{eqn:adjoint}
-\partial_tp^\top = p^\top\partial_zf(z(t),\theta(t))\,,
\end{equation} with the final condition
\begin{equation}\label{eqn:final_gradient}
p^\top(t=1) = \revadd{\langle} (g(z^\theta(1;x)) - y) \revadd{,} \nabla g(z^\theta(1;x)) \revadd{\rangle_{\mu}}\,.
\end{equation}1
\end{proposition}

It is worth noting that the functional derivative~\eqref{eqn:gradient_functional} requires solutions from both the forward equation ($z(t)$) and that of the adjoint equation ($p(t)$). The adjoint equation is solved backward in time, with the final condition~\eqref{eqn:final_gradient} coding the mismatch. In a sense, the adjoint equation tracks back the dependence of the mismatch to the configuration of $\theta(t)$.

Similarly, we can train ODE-nets, and this amounts to finding $\Theta$ using the data pair $(x,y)$ distributed according to $\mu$. We define the loss function as follows, and it is the counterpart of~\eqref{eqn:def.Etilde}:
\begin{equation}
\label{eqn:def.E}
\begin{aligned}
& E(\Theta) = \frac{1}{2}\|g(\rz_L^\Theta(x))-y\|_\mu^2\,, \\
& \mbox{where $\rz_l^{\Theta}(x)$ solves \eqref{eqn:leapfrog} or~\eqref{eqn:ERK-def}.}
\end{aligned}
\end{equation}
\revdel{Gradient-decent}{Gradient-descent} for this problem leads to an update of:
\[
\Theta^\new=\Theta^\old - \eta \nabla_{\Theta} E(\Theta^\old)\,.
\]
Typically we rely on auto-differentiation for the computation of $\nabla_{\Theta}E$. For the specific ODE-net structure discussed in this paper, we can lean on the vector-Jacobian product (VJP) formula used in reverse-mode auto-differentiation and have the leisure to spell out  $\nabla_{\Theta} E(\Theta)$ explicitly.
\revadd{The following proposition provides an explicit discrete expression for the gradient with respect to each parameter $\rtheta_l$, 
showing how the auto-differentiation computation unfolds layer by layer.}
\begin{proposition}\label{prop:grad_dis}
Let $E$ be defined as in~\eqref{eqn:def.E}, then the gradient $\nabla_\Theta E(\Theta) = \{\partial_{\theta_l} E\}$ and can be computed using the following:
\begin{itemize}
\item ODE-net defined through Leapfrog as in \eqref{eqn:leapfrog} has the following gradient component:
\begin{equation}\label{eqn:gradient}
L\partial_{\theta_l} E = \rp_l^\top \partial_{\rtheta} f(\rz_l, \rtheta_l)\,,\quad\forall l = 0,\cdots,L-1\,,
\end{equation}
where $\partial_\rtheta f$ is the Jacobian and $\rp$ is the \revdel{cotangent} \revadd{adjoint} vector that solves
\begin{subequations}\label{eqn:adjoint_dis}
\begin{align}
\rp_{L-1}^\top &= 2 \revadd{\langle} (g(\rz_L) - y) \revadd{,} \nabla g(\rz_L) \revadd{\rangle_{\mu}} \,, \label{eqn:adjoint_1}\\    
\rp_{L-2}^\top &= 2h \rp_{L-1}^\top \partial_z f(\rz_{L-1}, \rtheta_{L-1})\,,\label{eqn:adjoint_2}\\
\rp_l^\top &= \rp_{l+2}^\top + 2h \rp_{l+1}^\top\partial_z f(\rz_{l+1}, \rtheta_{l+1}) \notag \\
& \forall l=L-3,L-4,\dotsc,1\,, \label{eqn:adjoint_3} \\   
\rp_0^\top &= \frac{1}{2}\rp_{2}^\top + h \rp_{1}^\top\partial_z f(\rz_{1}, \rtheta_{1})\,.
\end{align}
\end{subequations}
\item ODE-net defined through 2-stage ERK as in \eqref{eqn:ERK-def} with $\nu = 2$ has the following gradient component:
\begin{subequations}\label{eqn:gradient-2-stage-ERK}
\begin{align}
& L \partial_{\rtheta_l} E = (1 - \frac{1}{2\alpha})\rp_{l+1}^\top \partial_\theta f(\rz_l, \rtheta_l) \notag \\
& \ \ \ \ \ \ \ + \frac{h}{2}\rp_{l+1}^\top \partial_z f(\rz_{l+\alpha}, \rtheta_{l+\alpha}) \partial_\theta f(\rz_l, \rtheta_l) \,,\label{eqn:gradient-2-stage-ERK_integer} \\
& L \partial_{\rtheta_{l+\alpha}} E = \frac{1}{2\alpha}\rp_{l+1}^\top \partial_\theta f(\rz_{l+\alpha}, \rtheta_{l+\alpha}) \label{eqn:gradient-2-stage-ERK_fraction} \,,
\end{align}
\end{subequations}
where $\rp$ is the \revdel{cotangent} \revadd{adjoint} vector that solves, for all $l$:
\begin{equation}\label{eqn:p-adjoint_2-stage-ERK}
\begin{aligned}
& \rp_l^\top  = \rp_{l+1}^\top + h \rp_{l+1}^\top \left(\left(1 - \frac{1}{2\alpha}\right) \partial_z f(\rz_l, \rtheta_l)\right. \\
& \ \ \ \ \ \ + \left.\frac{1}{2\alpha}\partial_z f(\rz_{l+\alpha}, \rtheta_{l+\alpha}) \right) \\
& \ \ \ \ \ \ + \frac{h^2}{2}\rp_{l+1}^\top \partial_z f(\rz_{l+\alpha}, \rtheta_{l+\alpha})\partial_z f(\rz_l, \rtheta_l)  \ , \\
&\text{with } \rp_L^\top = \revadd{\langle} (g(\rz_L) - y) \revadd{,} \nabla g(\rz_L) \revadd{\rangle_{\mu}} \, .
\end{aligned}
\end{equation}
Here the coefficients are chosen to be $b_1 = 1 - \frac{1}{2\alpha}$, $b_2 = \frac{1}{2\alpha}$, $c_1 = 0, a_{1,1} = c_2 = \alpha$ for some fixed $\alpha \in (0,1)$. $\alpha = \frac{1}{2}$ and $\frac{2}{3}$ respectively in the case of midpoint and Ralston.
\end{itemize}
\end{proposition}
The proof is deferred to Supplementary Sections~\ref{sec:appendix_proof_grad_dis_leapfrog} and~\ref{sec:appendix_proof_grad_dis_2_stage_ERK}. The Lagrangian multiplier is the major tool.

It is worth noting the similarities and differences between~\eqref{eqn:gradient_functional} for updating neural ODEs, and its counterparts~\eqref{eqn:gradient} and~\eqref{eqn:gradient-2-stage-ERK} for ODE-nets. Both~\eqref{eqn:gradient} and~\eqref{eqn:gradient-2-stage-ERK} assemble gradients as products of the ODE-net solution $\{\rz_l\}$ and the \revdel{cotangent} \revadd{adjoint} vector $\{\rp_l\}$, in a way very similar to that in~\eqref{eqn:gradient_functional}. It will also be shown below that~\eqref{eqn:adjoint_dis} and~\eqref{eqn:p-adjoint_2-stage-ERK} provide accurate approximations to~\eqref{eqn:adjoint}. These accurate ingredients do not guarantee an accurate final estimate of the gradient. For Leapfrog and 2-stage ERK respectively, it is the inaccurate final data preparation (in~\eqref{eqn:adjoint_1}) and the inaccurate coefficient (in~\eqref{eqn:gradient-2-stage-ERK}) that ultimately bring error to the computation of the gradient.

For higher-order ERK methods, the gradients become more intricate, but their overall structure remains consistent. 
\revdel{Moreover, the cotangent vector has an explicit form:} \revadd{The next proposition makes this structure explicit: it shows that even for high-order ERK schemes, the gradient can still be written in terms of the adjoint variable~$\rp$ that solves the discrete adjoint equation.}
\begin{proposition}\label{prop:p-adjoint-ERK}
Let the function $E(\Theta)$ defined as in~\eqref{eqn:def.E}, with $\rz_l$ solving high order ERK using~\eqref{eqn:ERK-def}. Then $\nabla_\Theta E$ is assembled using $\partial_\theta f$ and $\rp$, with $\rp$ being the \revdel{cotangent} \revadd{adjoint} vector that satisfies:
\begin{equation}\label{eqn:p-adjoint-ERK-general}
\begin{aligned}
& \rp_l^\top = \rp_{l+1}^\top \\
& + h \rp_{l+1}^\top (\mathsf{b}^\top \otimes \mathsf{I}_d) \partial_z f_l \cdot (\mathsf{I}_{\nu d} - h (\mathsf{A} \otimes \mathsf{I}_d)\partial_z f_l)^{-1} \cdot (\mathsf{e}_\nu \otimes \mathsf{I}_d) \\
& \text{with }  \rp_L^\top = \revadd{\langle} (g(\rz_L) - y) \revadd{,} \nabla g(\rz_L) \revadd{\rangle_{\mu}} \,.
\end{aligned}
\end{equation}
The specific assembling coefficients depend on $\mathsf{A}$ and $\mathsf{b}$.
\end{proposition}
We leave the proof to Supplementary Section~\ref{sec:appendix_proof_adjoint_ERK}.

\subsection{Oscillations in Auto-Differentiation}\label{sec:oscillation_autodiff}
We delve in for reasons to explain the oscillations observed in Figure~\ref{fig:failure_auto}. Since all derivations above are explicit, we can obtain a very clean comparison of gradients in the continuous and discrete settings. This amounts to a detailed reading of Proposition~\ref{prop:grad_dis}. Despite being long and tedious, this proposition reveals a lot of important information about the updates.

To start, we examine Leapfrog. Comparing~\eqref{eqn:gradient_functional} and~\eqref{eqn:gradient}, we immediately observe some nice features. $\partial_{\theta_l}E$ honors the structure of $\frac{\delta\calE}{\delta\theta}(lh)$. Both of them see the gradient as a product of $p$, the \revdel{cotangent} \revadd{adjoint} vector that runs from final time to $t=0$, and the Jacobian $\partial_\theta f$. Moreover, the \revdel{cotangent} \revadd{adjoint} vector of $\rp$ aligns very well with that of~\eqref{eqn:adjoint}. In particular,~\eqref{eqn:adjoint_3} is exactly a Leapfrog discretization of~\eqref{eqn:adjoint} for all $l\leq L-3$.

However, one strong drawback of the formula hinges on the definition of $\rp_{L-1}$ an $\rp_{L-2}$. Intuitively, for a good match, both $\rp_{L-1}$ and $\rp_{L-2}$ should approximately resemble $p(1)$. Comparing~\eqref{eqn:adjoint_1} and\revdel{~\eqref{eqn:adjoint}}\revadd{~\eqref{eqn:final_gradient}}, however, we see that $\rp_{L-1}$ nearly doubles the value of $p(1)$ while $\rp_{L-2}$ is $\calO(h)$. This immediately introduces $\calO(1)$ contrast at the beginning of the evolution of $\rp$. This high contrast propagates back in time using the Leapfrog formula~\eqref{eqn:adjoint_3}, bringing high oscillations throughout. This was seen in Figure~\ref{fig:failure_auto_1_leapfrog}. Moreover, refining the discretization by enlarging $L$ and reducing $h$ is not helpful. The $\calO(1)$ oscillation is encoded in the definition of $\rp_{L-1}$ and $\rp_{L-2}$, and it does not disappear in the $h\to0$ limit. \revadd{The entire derivations are completely exact and honor the computation of auto-differentiation that returns the honest derivatives for the discrete system. These oscillations are not artifacts induced by auto-differentiation itself, but rather the inconsistency embedded in the discrete system.}

We now examine ODE-Net generated by ERK-2 stage method. The situation here is even more intricate. This time, one can show that the \revdel{cotangent} \revadd{adjoint} vector $\rp_l$~\eqref{eqn:p-adjoint_2-stage-ERK} is actually a second order approximation of $p$\revdel{:} \revadd{and thus remains faithful to the continuous adjoint $p(t)$.}
\begin{proposition}\label{prop:2-stage-ERK-p-error}
Let $\rp_l$ be defined as in~\eqref{eqn:p-adjoint_2-stage-ERK}, and $p(t)$ be defined as in~\eqref{eqn:adjoint}.
Then \begin{equation}
|\rp_l - p(lh)| = \mathcal{O}(h^2) \, .
\end{equation}
\end{proposition}
The $\mathcal{O}$ notation contains dependence on $\mu$, the
bounds for $g$ and $\nabla g$, bounds on $\theta$ and its derivatives
up to the second order, the regularity constants for $f$, and the constant $\alpha$.

The proof for this proposition is straightforward, as presented in Supplementary Section~\ref{appendix_proof_apriori-2-stage-ERK}.
Essentially, one needs to define an auxiliary variable, termed $\hat{\rp}_l$, that solves the same equation~\eqref{eqn:p-adjoint_2-stage-ERK} with $\rz_l$ and $\rz_{l+\alpha}$ replaced by $z(lh)$ and $z((l+\alpha)h)$ respectively, and deploys the triangular inequality.

This proposition, together with~\eqref{eqn:2nd_order_z}, suggest that both the forward solver and the adjoint solvers are well-prepared. However, these ingredients are not correctly assembled. $\partial_{\theta_l}E$, according to~\eqref{eqn:gradient_functional}, is a simple product of $p$ and $\partial_{\theta}f$. In comparison, in~\eqref{eqn:gradient-2-stage-ERK}, the product of $\rp$ and $\partial_{\theta}f$ takes on special coefficients: At the integer point $\theta_l$,~\eqref{eqn:gradient-2-stage-ERK_integer} returns a product of $\rp$ and $\partial_\theta f$ with a factor of $1-\frac{1}{2\alpha}$. This factor is not $1$ for any choice of $\alpha \in (0,1)$, making it impossible to match $\partial_{\theta_l}E$ and $\frac{\delta\calE}{\delta\theta}(lh)$. At middle stage for $\theta_{l+\alpha}$, the factor becomes $\frac{1}{2\alpha}$, and this would agree with coefficient $1$ only if $\alpha = \frac{1}{2}$, the case for the Midpoint method. As a result, the Midpoint update agrees with that of the continuous setting at half-integer stages but not at integer stages, and Ralston is completely wrong for all stages. These arguments confirm our finding in Figure~\ref{fig:failure_auto_1_midpoint} and~\ref{fig:failure_auto_1_ralston}. 

Higher order ERK suffers from the same challenge. Though the assembling coefficient vary, the typical situation nevertheless gives incorrect match to its continuous versions, and provides high oscillations. For example, in the case of \revdel{Nystrom} \revadd{Nystr\"{o}m}~\eqref{eqn:Nystrom-def},
$$\begin{aligned}
\partial_{\theta_l}E & \sim \frac{1}{4} \rp_{l+1}^\top \partial_\theta f(\rz_l, \rtheta_l) \,, \\
\text{and}\quad\partial_{\theta_{l,2/3}}E & \sim \frac{3}{4}\rp_{l+1}^\top \partial_\theta f(\rz_l, \rtheta_l) \, .
\end{aligned}$$
The factors $\frac{1}{4}$ and $\frac{3}{4}$ are the reasons of the oscillations and reduction of true gradient, as shown in Figure~\ref{fig:failure_auto_1_nystrom}.

\section{Modified auto-differentiation for ODE-net}\label{sec:correction}

The discussion above pinpoints the root of the failure in the computation of the gradients, also shines lights on possible fixes. Considering that auto-differentiation is widely used and easy to implement, serving as a blackbox tool, we would like to develop methods that are non-intrusive, and optimally extract information that is already in the auto-differentiation computation.

The methods provided below require some generic assumptions for the fix to hold.
\begin{assumption}\label{assumption:main}
The activation function $f(z, \theta)$ satisfies:
a) Regularity: $f$ is third order continuously differentiable with derivatives bounded by $C_a$; 
b) Quadratic-linear growth: For all $z \in \mathbb{R}^d, \theta \in \mathbb{R}^{n}$, $|f| \leq C_b (|\theta|^2+1)(|z| + 1)$ for some $C_b$.
The output function $g(z)$ to possess:
c) Regularity: $g(z)$ and its gradient $\nabla g(z)$ are Lipschitz continuous with Lipschitz constants bounded by $C_c$.
The training data distribution $\mu$ has:
d) Compact Support: There exists some $C_d > 0$ such that $\text{supp}(\mu)$, the support of $\mu$, lies within a ball centered at the origin with radius $C_d$.
\end{assumption}
These assumptions are rather mild and are typically satisfied. For example, the commonly used activation function
$f(z;\sigma,W,b) = \sigma \cdot \tanh(W z + b)$ and the output function $g(z) := z$ \revdel{satisfiy} \revadd{satisfy} all the conditions above.

\subsection{Modified Leapfrog ODE-net}\label{sec:LMM}

We first discuss the fix for updating coefficients in Leapfrog. The observation in Section~\ref{sec:oscillation_autodiff} suggests that the updating procedure of $\rp$ in the decreasing order of $l$ honors the underlying ODE~\eqref{eqn:adjoint}, but its final state preparation introduces high oscillations. From here, it is expected that a certain type of averaging would smooth out the oscillation while continue honoring the underlying flow updates. To this end, we define a modification matrix
$\rT\in\mathbb{R}^{nL\times nL}$ composed of diagonal blocks, each a
multiple of identity matrix $I\in\mathbb{R}^{n\times n}$. Specifically, we have $\rT = \{\rT_{i,j}\}_{i,j=0}^{L-1}$ with $\rT_{0,0}=I$, $\rT_{0,1} =
\frac{3}{4}I$, $\rT_{0,3}=-\frac{1}{4}I$, $\rT_{1,0}=\frac{1}{2}I$ and

\begin{equation}\label{def:T_matrix}
\rT_{i,j} = \left\{
\begin{array}{llll}
\frac{1}{4}I & \text{for } 2 \leq i \leq L-1, j = i-1 \\
& \text{ and } 1 \leq i \leq L-2, j = i + 1 \ ; \\
\frac{1}{2}I & \text{for } 1 \leq i = j \leq L - 1 \,.
\end{array}\right.
\end{equation}
All blocks not specified are simple $0$s. We then define the modified gradient:
\begin{equation}\label{eqn:modified_gradient}
\begin{aligned}
\bar\nabla_\Theta E & = \rT\nabla_\Theta E\quad \\
\text{or equivalently}\quad
\bar{\partial}_{\theta_i} E & = \sum_j \rT_{i,j}\partial_{\theta_j}E\,.
\end{aligned}
\end{equation}
The $[\frac{1}{4},\frac{1}{2},\frac{1}{4}]$ pattern suggests that the modified gradient is a weighted average over the three neighboring stages of the gradient $\nabla_\Theta E$. Noting that $\nabla_\Theta E$ is the output of auto-differentiation, so our fix manipulates this quantity as a whole, and serves as a post-processing step. Considering the highly oscillatory pattern of the gradient, this strategy of smoothing neighboring grids should help averaging out the oscillation, and provides more accurate gradients. This is justified in the following theorem:
\begin{theorem}\label{theorem:LMMmain}
Let $\theta(t)$ be a second order continuously differentiable function
with bounded second derivative, and denote
$\Theta = \{\theta_l\}_{l=0}^L$ with $\theta_l =
\theta(t_l)\in\mathbb{R}^n$. Then under mild conditions on $f$
(discussed below), the modified gradient is a second order
approximation to the Fr\'echet derivative evaluated at discrete time
steps. That is, for any $l$,
\[
\left|(L \bar\nabla_\Theta E)_l - \left.\frac{\delta \calE}{\delta {\theta}}\right|_{\theta(t)}(lh)\right| = \mathcal{O}(h^2)\,.
\]
\end{theorem}
The $\mathcal{O}$ notation contains dependence on $C_a, C_b, C_c, C_d$ and bounds on $\theta$ and its derivatives up to the second order.
Here $C_a, C_b, C_c, C_d$ are defined as in Assumption~\ref{assumption:main}.

The core of the proof lies in the realization that ``averaging'' the gradient is effectively equivalent (up to $\mathcal{O}(h^2)$ error) to ``averaging'' the \revdel{cotangent} \revadd{adjoint} vector $\rp$. The averaging for the gradient is already denoted by $\rT$ defined in~\eqref{def:T_matrix}, and we are to define a new matrix $\tilde{\rT}$ (see definition in~\eqref{def:T_tilde_matrix}) that acts on $\{\rp\}$. Two auxiliary quantities will be defined. Firstly, we define $\hat{\rp} = \{\hat{\rp}_l\}_{l=0}^{L-1}$, a semi-continuous object that uses the same formula~\eqref{eqn:adjoint_dis} with $\rz_l$ replaced by $z(lh)$. We then define
an averaged version:
\begin{equation}\label{eqn:def_hat_q}
\hat{\sfq} = \tilde{\rT}\hat{\rp}\,,\quad\text{or equivalently}\quad \hat{\sfq}_l = \sum_{k = 0}^{L-1}\tilde{\rT}_{l,k}\hat{\rp}_k \,.
\end{equation}
Accordingly, the gradients defined using these two auxiliary quantities are:
\begin{equation}\label{eqn:def-grad-p-q-leapfrog}
\begin{aligned}
(L\nabla_{\Theta}^\rp E)_l & := \hat{\rp}_l^\top \partial_\theta f(z(lh), {\rtheta}_l)\,, \\
(L\nabla_{\Theta}^\sfq E)_l & := \hat{\sfq}_l^\top \partial_\theta f(z(lh), {\rtheta}_l)\,,
\end{aligned}
\end{equation}
for $l = 0, 1, 2, \dots, L-1$. The proof follows a triangle inequality:

\begin{proof}[Proof of Theorem~\ref{theorem:LMMmain}]
We first decompose, using triangle inequality, for all $l$:
\begin{equation}\label{eqn:LMMmain-proof-main-estimate}
\begin{aligned}
& \ \ \ \ \ \ \left| (L \rT\nabla_\Theta E)_l - \left.\frac{\delta \calE}{\delta {\theta}}\right|_{\theta(t)}(lh)\right| \\
& \leq \underbrace{\left| (L \rT \nabla_\Theta E)_l - (L \rT \nabla_\Theta^\rp E)_l \right|}_{\text{Term I}} + \underbrace{\left|( L \rT \nabla_\Theta^\rp E)_l - (L\nabla_\Theta^\sfq E)_l \right|}_{\text{Term II}} \\
& \ \ \ \ \ \ + \underbrace{\left|(L\nabla_\Theta^\sfq E)_l - \left.\frac{\delta \calE}{\delta {\theta}}\right|_{\theta(t)}(lh)\right|}_{\text{Term III}}\,.
\end{aligned}  
\end{equation}
The theorem is proved if all three terms are $\mathcal{O}(h^2)$. This can be proved to hold. We discuss the proof strategy in Appendix~\ref{sec:appendix_strategy_proof_LMM}, and leave details to supplementary materials in Corollary~\ref{cor:term1}, Lemma~\ref{lem:averaging-on-grad} and Corollary~\ref{cor:term3} respectively.
\end{proof}

\subsection{Modified 2-stage ERK ODE-nets}\label{sec:ERK}
We now switch to the discussion of 2-stage ERK ODE-nets. Similar to deploying a post-processing to correct the derivatives provided by auto-differentiation for Leapfrog, we can also design a post-processing to fix the gradient of coefficients in a 2-stage ERK ODE-net.

The challenge here is different from that of Leapfrog. As observed in Section~\ref{sec:oscillation_autodiff}, in this situation, the updates of $\rp$ still resembles the structure of ODE for $p(t)$, and oscillations are introduced mainly because of the wrong coefficients in the assembling stage. As a consequence, one only needs to adjust assembling coefficients. It is easily done on the integer points. For $l = 0, 1, \dots, L-1$, we propose
\begin{equation}\label{eqn:2-stage-ERK-mod}
\overline{\partial}_{\rtheta_{l}} E = \rp_l^\top \partial_\theta f(\rz_l, \rtheta_l)\,.
\end{equation}
At the noninteger point, we first define the interpolation:
\begin{equation}\label{eqn:q-def-2-stage-ERK}
\sfq_{l + \alpha}^\top := (1 - \alpha)\rp_l^\top + \alpha \rp_{l+1}^\top \, ,
\end{equation}
and then set:
\begin{equation}\label{eqn:2-stage-ERK-mod_alpha}
\overline{\partial}_{\rtheta_{l + \alpha}} E = \sfq_{l + \alpha}^\top \partial_\theta f(\rz_{l + \alpha}, \rtheta_{l + \alpha}) \,.
\end{equation}

We should note that $\{\rp_l\}$ \revdel{is an} \revadd{defined by~\eqref{eqn:p-adjoint_2-stage-ERK} coincide with the} output of auto-differentiation \revadd{$\{\partial_{\rz_l} E\}$}, as justified by the following proposition:
\begin{proposition}\label{prop:autodiff-p-2-stage-ERK}
For all $l = 0,1, \dots, L$, the solution $\rp_l$ to ~\eqref{eqn:p-adjoint_2-stage-ERK} satisfies: $\{\partial_{\rz_l} E\}$.
\end{proposition}
The proof is straightforward calculation and is in Supplementary Section~\ref{appendix_proof-autodiff-p-2-stage-ERK}. 

This proposition suggests that by running auto-differentiation for $E$ on $\rz_l$, we obtain $\{\rp_l\}$ that can be directly assembled in our modified gradient~\eqref{eqn:2-stage-ERK-mod} and~\eqref{eqn:2-stage-ERK-mod_alpha}. Therefore, our method is still a post-processing strategy. Since the assembling coefficients are correct, we expect this modified gradient is a better estimate. This is justified in the following theorem.
\begin{theorem}\label{theorem:ERK-correction}
Under the same conditions as in Theorem~\ref{theorem:LMMmain}, the proposed fix~\eqref{eqn:2-stage-ERK-mod}-\eqref{eqn:2-stage-ERK-mod_alpha}, as a post-processing of auto-differentiation results, provide a second order approximation to the gradient defined on the continuous setting:
\begin{subequations}\label{eqn:ERK-correction-error}
\begin{align}
\left\vert
\overline{\partial}_{\rtheta_{l+\alpha}} E
- \frac{\delta \calE}{\delta \theta}\bigg\vert_{\theta(t)}((l+\alpha)h) 
\right\vert & = \mathcal{O}(h^2)\,, \\
\left\vert \overline{\partial}_{\rtheta_{l}} E
-  \frac{\delta \calE}{\delta \theta}\bigg\vert_{\theta(t)}(lh) \right\vert & = \mathcal{O}(h^2) \  .
\end{align}
\end{subequations}
\end{theorem}
The $\mathcal{O}$ notation contains dependence on $C_a, C_b, C_c, C_d$, bounds on $\theta$ and its derivatives up to the second order, and the specific parameter $\alpha$ that determine the precise structure of the ERK method. The constants here are defined as in Assumption~\ref{assumption:main}.

We leave the discussion of the proof strategy to Appendix~\ref{sec:appendix_strategy_proof_ERK}, with details provided in Supplementary Section~\ref{sec:appendix_2-stage-ERK}. The major component of the proof is Taylor approximation both for the forward and the backward equation to the second order.

\section{Numerical experiments}\label{sec:numerical_examples}

We present numerical experiments to verify our observations, and validate the correction method we proposed.
The implementation details and additional figures are deferred to Appendix~\ref{sec:appendix_exp-details} and Supplementary Section~\ref{sec:appendix_additional-figures} respectively.

Four discretization schemes are studied, and they are Leapfrog~\eqref{eqn:leapfrog}, an example of LLM, Midpoint and Ralston~\eqref{eqn:midpoint}-\eqref{eqn:ralston}, as examples of ERK-2 stage, and \revdel{Nystrom} \revadd{Nystr\"{o}m}~\eqref{eqn:Nystrom-def} as an example of ERK-3 stage respectively. For the first three schemes, we provide corrections.

\subsection{Gradient computation}\label{sec:numerical_examples-grad-comp}

\begin{figure}[ht]
\begin{center}
    \includegraphics[width=0.475\textwidth]{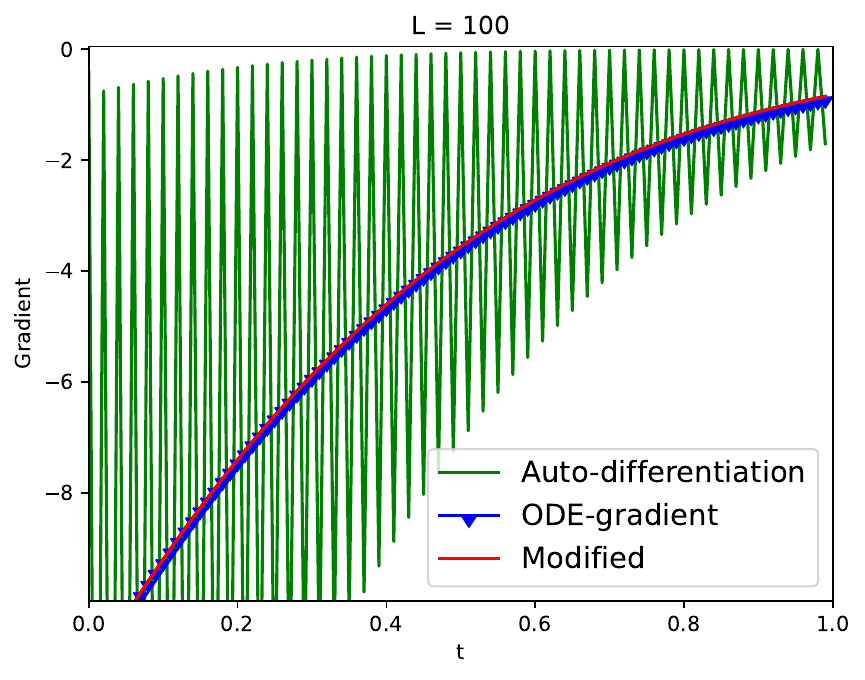}
    \includegraphics[width=0.475\textwidth]{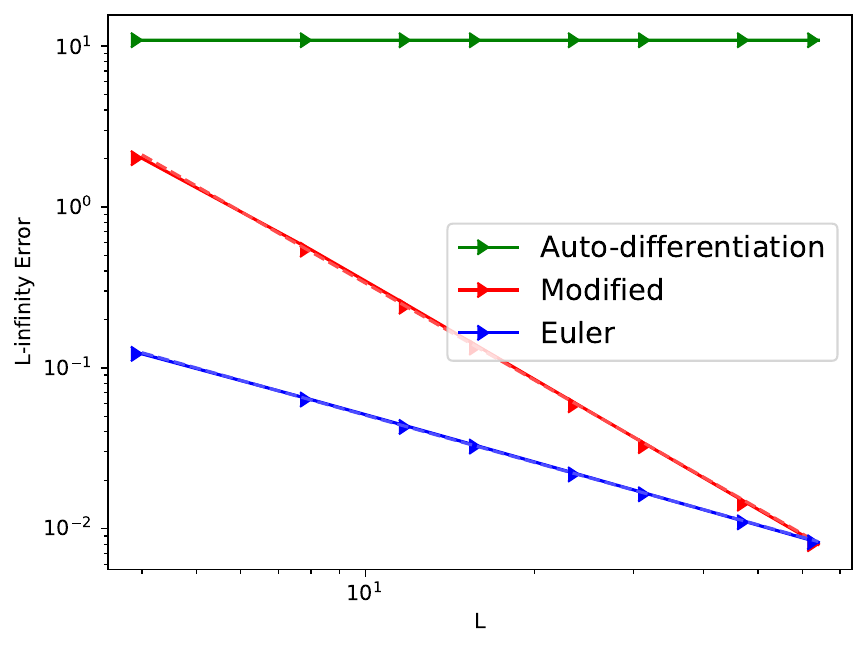}
\caption{Gradient computation of the Leapfrog Method. 
The left panel compares $L \nabla_\Theta E$ (green) with $\NN_\Theta$ formulated in~\eqref{eqn:leapfrog}
      with $\frac{\delta\calE}{\delta\theta}$ (blue). The auto-differentiation outputs $L \nabla_\Theta E$ that
      oscillates around the true gradient. Our modification (red) yields gradients that match the
      ground-truth. In the right panel, we show the decay of error as $h$ shrinks. Since the oscillations do not diminish as $h\to 0$, the $L_\infty$ norm of the
      gradient difference stays as a constant. The convergence rates for the modified gradient and
      that computed from forward Euler are $2.00$ and $0.98$ respectively, confirming the theorem.}
\label{fig:failure_auto_leapfrog}
\end{center}
\end{figure}

In Figure~\ref{fig:failure_auto_leapfrog},~\ref{fig:failure_auto_midpoint}, and~\ref{fig:failure_auto_Ralston}, we plot the performance of auto-differentiation result and the our proposed correction for Leapfrog, Midpoint and Ralston respectively. In each plot, the left panel shows comparison of true gradient (blue), gradient as an output of auto-differentiation (green) and our modified correction (red), following~\eqref{eqn:modified_gradient} and~\eqref{eqn:2-stage-ERK-mod}. The right panel demonstrates the convergence rate with respect to $h$. 
For all three methods, the output of auto-differentiation are highly oscillatory and non-converging: In $L_\infty$ norm, the error stays as a constant for all $h$. The modified gradients (red) visually show no difference from the true gradient, and the convergence rate with respect to $h$ are $h^{2.00}$, $h^{1.70}$, and $h^{1.84}$ respectively, confirming the second order accuracy predicted in Theorem~\ref{theorem:LMMmain} and~\ref{theorem:ERK-correction}. As a comparison, we also plot out the convergence using the standard forward Euler discretization. The rate stays around $h^{~0.98}$, and thus confirms it as a first order method.

\subsection{Curve-learning}\label{sec:numerical_examples-curve-learning}

\begin{figure}[ht]
\begin{center}
    \includegraphics[width=0.75\textwidth]{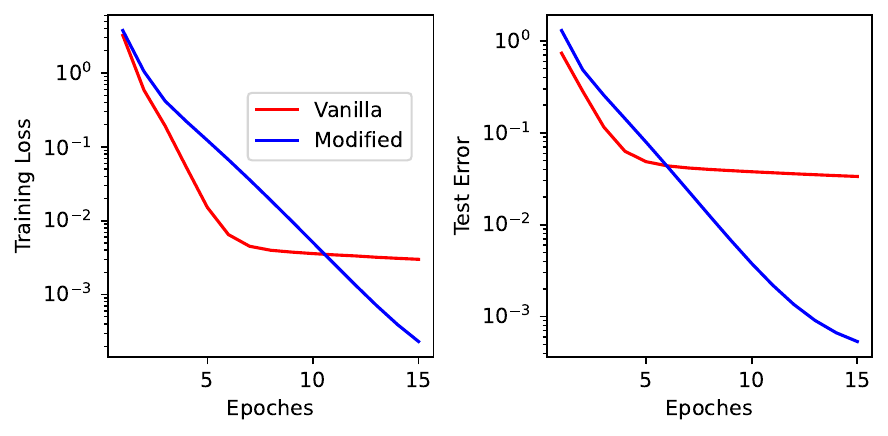} \\
    \includegraphics[width=0.75\textwidth]{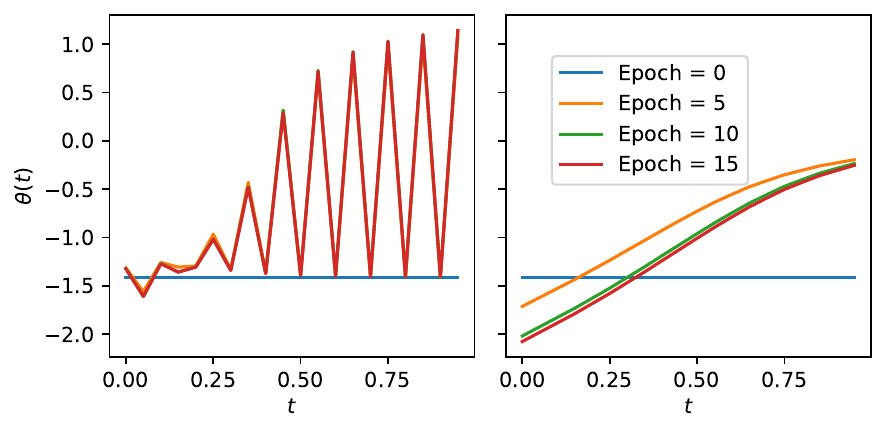} \\
    \includegraphics[width=0.6\textwidth]{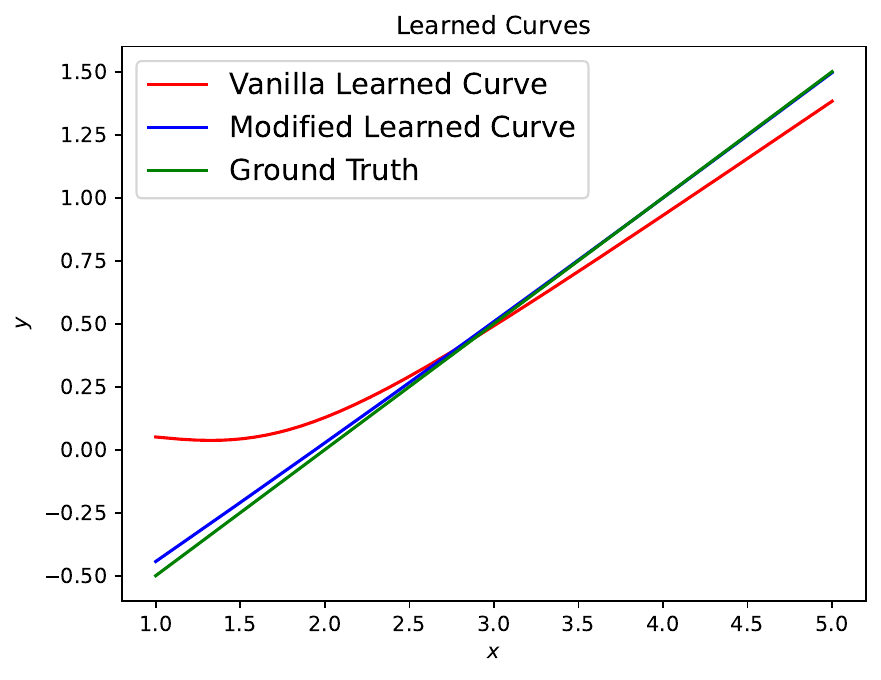}
\caption{
For Leapfrog method. 
The top row shows the training loss and test error over epochs. 
The middle row compares $\Theta$ outputs from vanilla \revadd{(unmodified)} auto-differentiation and our modified gradients\revdel{. The}\revadd{, 
and the} bottom plot presents the final reconstruction. 
\revadd{Quantitatively, the corrected gradients achieve much lower final errors:
training loss decreases from $3.0\times10^{-3}$ to $2.3\times10^{-4}$, 
and test error from $3.3\times10^{-2}$ to $5.4\times10^{-4}$. 
These improvements confirm the effectiveness of the proposed correction.}
}
\label{fig:Leapfrog_Correction}
\end{center}
\end{figure}

\begin{figure}[ht]
\begin{center}
\includegraphics[width=0.8\textwidth]{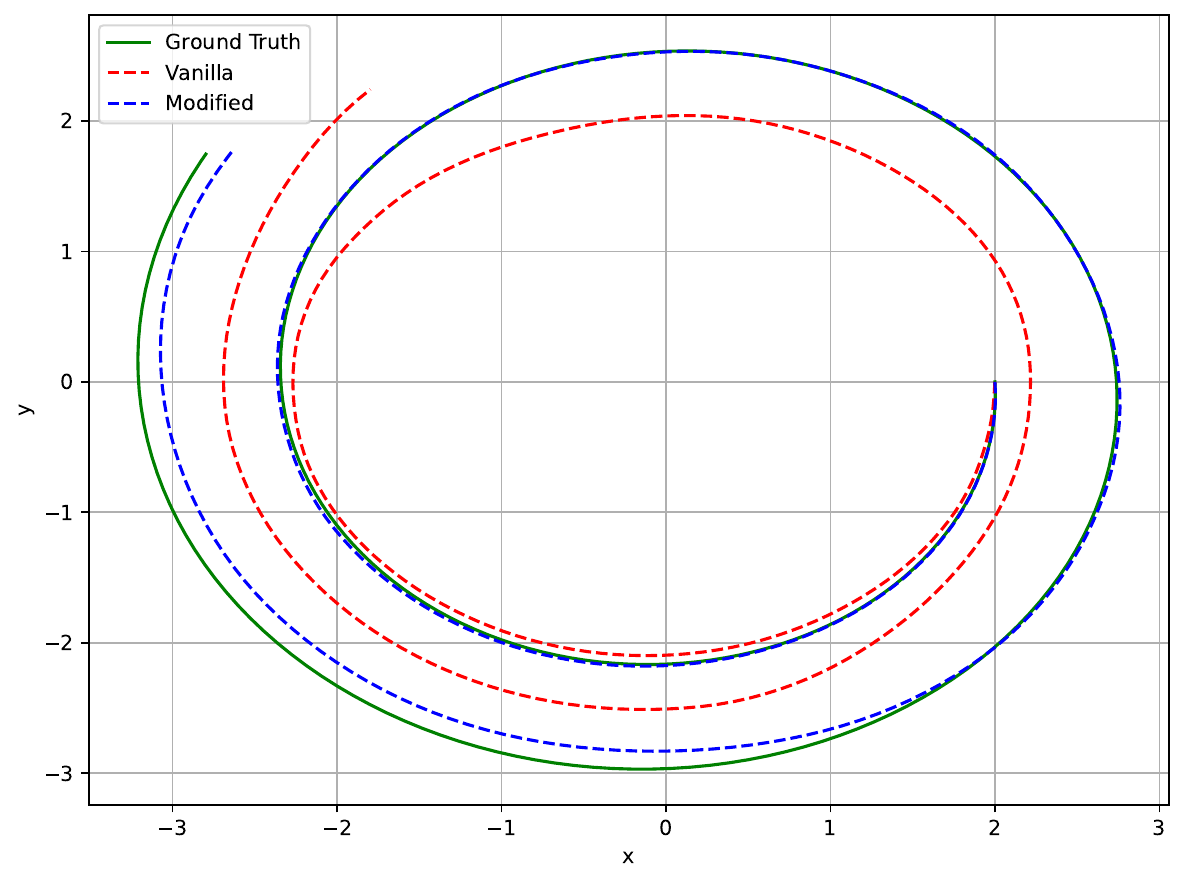}
\caption{For Leapfrog Method. The figure compares reconstructions of a spiral learned using auto-differentiation gradients (red) and our modified gradients (blue). Our modification achieves results that align more closely with the ground truth (green).
\revadd{Quantitatively, the trajectory Root Mean Square Error (RMSE) drops from $0.75$ to $0.07$, 
the final-point error from $1.06$ to $0.13$, 
and the maximum deviation from $1.07$ to $0.15$. 
These results confirm that the corrected gradients yield trajectories that 
closely follow the ground-truth dynamics.}}
\label{fig:Leapfrog_Spiral}
\end{center}
\end{figure}

To further demonstrate the effectiveness of our proposed corrections, we present the training results using
auto-differentiation gradients and our modified gradients.
For a fair comparison, in each experiment, ODE-net training procedures using auto-differentiation gradients were compared to those using our modified gradients under identical initialization, optimization settings, and step
sizes.

We first experimented on learning a linear function.
The improvement for Leapfrog, Midpoint and Ralston are respectively shown in Figure~\ref{fig:Leapfrog_Correction},~\ref{fig:Midpoint_Correction}, and~\ref{fig:Ralston_Correction}. 
The top rows show both auto-differentiation updates (in red) and our modified gradients (in blue) decay in error over epochs.
In these examples, our test error clearly shows better performance compared to auto-differentiation. 
The second rows depict the computed $\Theta$. 
While auto-differentiation leads to strong oscillatory patterns, our modified gradients provide smoother parameters and better convergence in \revdel{epoches} \revadd{epochs}. 
The last rows present the final function reconstructions.
For a linear function, our modified gradients show consistent improvement toward the ground truth (in green).

We also tested learning a spiral in $\mathbb{R}^2$.
The results for Leapfrog, Midpoint, and Ralston methods are shown in Figures~\ref{fig:Leapfrog_Spiral}, \ref{fig:Midpoint_Spiral}, and~\ref{fig:Ralston_Spiral}, respectively.
In each plot, the reconstructed trajectory using auto-differentiation is shown in red, while the reconstruction using the modified gradient is shown in blue.
Our modification produces results that more closely match the ground truth, shown in green.

\section{Conclusions}
\label{sec:conclusions}

In this paper, we investigated the oscillation phenomena that occur when training Leapfrog and Runge-Kutta ODE-Net models, and proposed methods to correct them.
Our work advances the mathematical understanding of the comparison between the Optimize-Then-Discretize and Discretize-Then-Optimize approaches, and also contributes to a better understanding of auto-differentiation as a tool in machine learning.

Our findings suggest that careful treatment of discretization and optimization order is crucial for improving stability and accuracy in training physics-inspired neural networks.
\revadd{The proposed post-processing framework in Section~\ref{sec:ERK} can in principle be extended to higher-order ERK schemes, including the Nystr\"{o}m scheme illustrated in Figure~\ref{fig:failure_auto_1_nystrom}. 
Under higher regularity assumptions, the auto-differentiated adjoint variables remain high-order approximations of the continuous adjoint, and the main modification lies in constructing higher-order interpolants for intermediate states. 
While the derivations are algebraically involved, the extension is conceptually straightforward.}
Future work may explore extending our approach to more general stiff systems or higher-order methods, and further improving gradient approximations in large-scale settings.

\appendix

\section{Proof Strategy of Theorems~\ref{theorem:LMMmain} and~\ref{theorem:ERK-correction}}\label{sec:appendix_strategy_proof}

The proof for the two major theorems (Theorems~\ref{theorem:LMMmain} and~\ref{theorem:ERK-correction}) can be technical, but the strategies used for the proof are relatively straightforward. We summarize them below, and leave details to the supplementary materials.

\subsection{Proof Strategy for Theorem~\ref{theorem:LMMmain}}\label{sec:appendix_strategy_proof_LMM}

This is the theorem that states the averaged (post-processed) autodifferentiation of the neural ODEs ran by Leapfrog approximates the true functional gradient. As discussed in Section~\ref{sec:LMM}, the proof of Theorem~\ref{theorem:LMMmain} hinges on controlling the three terms in~\eqref{eqn:LMMmain-proof-main-estimate}, using a-priori estimates established earlier.

\textbf{Term I} is addressed in Corollary~\ref{cor:term1}, where we compare the discrete gradient $\nabla_\Theta^\rp E$ (defined in~\eqref{eqn:def-grad-p-q-leapfrog}) with the exact gradient $\nabla_\Theta E$ (from~\eqref{eqn:gradient}). 
These expressions differ only by replacing $\rp_l$ with $\hat{\rp}_l$ and $\rz_l$ with $z(lh)$, and this discrepancy is shown to be $\mathcal{O}(h^2)$, based on the second-order accuracy and boundedness of the Leapfrog integrator, as established in Lemmas~\ref{lem:Delta_z} and~\ref{lem:Delta_p}. The result also relies on the boundedness of $\rp_l$, $\hat{\rp}_l$, and the second-order regularity of $f$.

\textbf{Term II} is analyzed in Lemma~\ref{lem:averaging-on-grad}, where we compare $L \rT \nabla_\Theta^\rp E$ with $L \nabla_\Theta^\sfq E$. A direct computation using the structure of the averaging matrix $\rT$ (defined in~\eqref{def:T_matrix}) shows that reordering the averaging introduces only an $\mathcal{O}(h^2)$ error. The proof uses the boundedness of $\hat{\rp}_l$, the third-order regularity of $f$, and second-order bounds on $\theta$.

\textbf{Term III} is handled in Corollary~\ref{cor:term3}, which compares $\nabla_\Theta^\sfq E$ with the functional gradient $\left.\frac{\delta \calE}{\delta \theta}\right|_{\theta(t)}(t)$ (defined in~\eqref{eqn:gradient_functional}). 
The argument is based on Lemma~\ref{lem:averaging-on-p}, which shows that the averaged quantity $\hat{\sfq}_l$ closely approximates $p(lh)$. Although $\hat{\rp}_l$ deviates from the Leapfrog update near initialization, averaging corrects this effect. The proof also uses the boundedness of $z$ and the first-order regularity of $f$.

Combining these three estimates yields the desired $\mathcal{O}(h^2)$ bound in~\eqref{eqn:LMMmain-proof-main-estimate}, completing the proof.

\subsection{Proof Strategy for Theorem~\ref{theorem:ERK-correction}}\label{sec:appendix_strategy_proof_ERK}

As discussed in Section~\ref{sec:ERK}, the proof of Theorem~\ref{theorem:ERK-correction} compares the discrete gradient in the 2-stage ERK scheme (defined in~\eqref{eqn:2-stage-ERK-mod}) with the exact functional gradient~\eqref{eqn:gradient_functional}, both at integer time points and intermediate stages.

The discrepancy is estimated via triangle inequalities and Taylor expansions, and shown to be $\mathcal{O}(h^2)$ under suitable regularity assumptions.

The argument relies on two components:

\begin{itemize}
    \item \textbf{A-priori bounds} on $p(t)$, $z(t)$, $\rz_l$, and $\hat{\rp}_l$, established in Lemma~\ref{lem:2-stage-ERK-bound}.
    
    \item \textbf{Local consistency estimates}:
    \begin{itemize}
        \item Lemma~\ref{lem:2-stage-ERK-Delta_z} and Corollary~\ref{cor:2-stage-ERK-f-theta-error} show that $\rz_l$ and $\rz_{l+\alpha}$ approximate $z(lh)$ and $z((l+\alpha)h)$ up to $\mathcal{O}(h^2)$, due to the second-order accuracy of the 2-stage ERK method.
        
        \item Proposition~\ref{prop:2-stage-ERK-p-error} ensures that $\rp_l$ approximates $p(lh)$ up to $\mathcal{O}(h^2)$, with supporting bounds from Lemmas~\ref{lem:2-stage-ERK-delta-p-init} and~\ref{lem:2-stage-ERK-delta-p-hat}.
        
        \item Corollary~\ref{cor:middle_stage_q} controls the middle-stage interpolation $\sfq_{l+\alpha}$, whose accuracy follows from the linear interpolation structure and Proposition~\ref{prop:2-stage-ERK-p-error}.
    \end{itemize}
\end{itemize}

Together, these results yield the final $\mathcal{O}(h^2)$ error estimate, completing the proof.

\section{Experimental details}\label{sec:appendix_exp-details}

In this section, we describe more details about our numerical experiments in Section~\ref{sec:numerical_examples}.
All ODE-nets, including Leapfrog, Midpoint, Ralston, and \revdel{Nystrom} \revadd{Nystr\"{o}m} are implemented using PyTorch's \texttt{torch.nn} module, with gradient directly provided by Pytorch's autodifferentiation. 

\subsection{Details of gradient computation experiment}\label{sec:appendix_exp-details-gradient-computation}

The gradient computation experiments in Section~\ref{sec:numerical_examples-grad-comp} considers a simple case, where for all of Leapfrog (Figure~\ref{fig:failure_auto_1_leapfrog} and Figure~\ref{fig:failure_auto_leapfrog}), Midpoint (Figure~\ref{fig:failure_auto_1_midpoint} and Figure~\ref{fig:failure_auto_midpoint}), Ralston (Figure~\ref{fig:failure_auto_1_ralston} and Figure~\ref{fig:failure_auto_Ralston}) and \revdel{Nystrom} \revadd{Nystr\"{o}m}(Figure~\ref{fig:failure_auto_1_nystrom}) we used the initial condition $x = 3$, label $y = 24$, output function $g$ being the identity function, and the loss function be the mean square loss~\eqref{eqn:def.E}.
The initial parameter curve $\theta : [0,1] \to \mathbb{R}^3$ is defined to be $\theta(t) := (\frac{t+2}{4},0,1)$, and the activation function is $f(z,\theta) := \theta_3\tanh(z \theta_1 + \theta_2)$, where $\theta_1, \theta_2, \theta_3$ denote the first, second and third coordinate of $\theta$ respectively.
For simplicity, we only compute and plot the first coordinate of the gradients.
We demonstrate the convergence of the $L_\infty$ error of our modified gradients using different $L$-s ranging from $4$ to $64$.
The gradients in the Euler ODE-Net model are computed using the closed form expression in \revdel{\cite{DiChLiWr:2022overparameterization}}\revadd{\cite[Appendix C.4]{DiChLiWr:2022overparameterization}}.
In addition, the ``ODE-gradients,'' which are the values of the functional derivatives $\left.\frac{\delta \calE}{\delta \theta}\right|_{\theta(t)}(t)$ at the corresponding time steps, are computed using~\eqref{eqn:gradient_functional}, where for $z(t)$ and $p(t)$ we solve them using Scipy's \texttt{scipy.integrate.solve\_ivp} function.

\subsection{Details of curve learning experiment}\label{sec:appendix_exp-details-curve-learning}

The first curve learning experiment in Section~\ref{sec:numerical_examples-curve-learning} aims at learning a linear function on $\mathbb{R}\to\mathbb{R}$ defined as $x \mapsto \frac{x}{2}-1$ using Leapfrog (Figure~\ref{fig:Leapfrog_Correction}), Midpoint (Figure~\ref{fig:Midpoint_Correction}) and Ralston (Figure~\ref{fig:Ralston_Correction}) ODE-net models.
We take the activation function to be $f(z,\theta) := \theta_3\tanh(z \theta_1 + \theta_2)$, where $\theta_1, \theta_2, \theta_3$ denote the first, second and third coordinate of $\theta$ respectively.
We also fix the output function $g$ being the identity function, and the loss function be the mean square loss~\eqref{eqn:def.E}.
We initiate all models at $\theta(t) := (-1.11, 0.33 , 1.41)$ and plotted the third coordinate of the parameter curves learned using both vanilla and modified gradient at training epochs from $0$ to $15$.
The training dataset consists of $128$ evenly sampled points for $x$ ranging from $2$ to $4$, and the test dataset consists of $63$ evenly sampled points for $x$ ranging from $1$ to $5$.
The numbers are chosen so that we would be able to test both the interpolation and extrapolation capabilities of the trained models.
We set $L = 20$ for Leapfrog and $L = 10$ for Midpoint and Ralston, so all of them will have $20$ stages in total.
For a fair comparison, we use PyTorch's \texttt{torch.optim.SGD} with identical rescaling methods applied to both inputs and labels across all experiments. 
We fix the random seed, use $64$ as the batch size, set the learning rate to $0.1$, and run the training process for $15$ epochs.

The second curve-learning experiments in Section~\ref{sec:numerical_examples-curve-learning} learns a spiral in  $\mathbb{R}^d$, whose dynamic is defined by
\begin{equation}\label{eq:spiral}
\frac{d}{dt} \begin{bmatrix} x \\ y \end{bmatrix} =
\begin{bmatrix}
0.1 & 2.0 \\
-2.0 & 0.1
\end{bmatrix}
\begin{bmatrix} x \\ y \end{bmatrix},
\quad \text{with } \begin{bmatrix} x(0) \\ y(0) \end{bmatrix} = \begin{bmatrix} 2 \\ 0 \end{bmatrix}
\end{equation}
using Leapfrog (Figure~\ref{fig:Leapfrog_Spiral}), Midpoint (Figure~\ref{fig:Midpoint_Spiral}) and Ralston (Figure~\ref{fig:Ralston_Spiral}) ODE-net models.
We use the $(x(t),y(t))$ pairs as the inputs and the corresponding time-derivatives $(\frac{d x}{dt}(t),\frac{d y}{dt}(t))$ as the labels.
In each experiments, we use a randomly initiated linear mapping to bring the inputs to $\mathbb{R}^4$, train our ODE-net models in $\mathbb{R}^4$, and use another randomly initiated linear mapping to project the outputs down to $\mathbb{R}^d$.
The two linear mappings are fixed and not trainable.
We still use $\tanh$ as the activation function, and the mean square loss as the loss function.
We set $L = 40$ for Leapfrog and $L = 20$ for Midpoint and Ralston, so all of them will have $40$ stages in total.
We initiate all ODE-net layers to be the same randomly sampled parameters.
The training dataset is generated by running~\eqref{eq:spiral} from time $0$ to $5$ using Scipy's \texttt{scipy.integrate.solve\_ivp} function with a step size of $0.01$, so that at each time step, the state and the corresponding time derivatives yield a data/label pair. 
We run~\eqref{eq:spiral} on $[0,5]$ again with the learned models to generate the reconstructions.
All optimizations are done with PyTorch's \texttt{torch.optim.SGD} optimizer, fixed random seed, batch size of 64, learning rate of 0.02, and trained for 200 epochs.

\section*{Acknowledgments}

We acknowledge the use of GPT-4o to assist with editing and polishing the authors' written text.

\supplement
\section{Proof of Proposition~\ref{prop:grad_dis} and~\ref{prop:p-adjoint-ERK}}\label{sec:appendix_proof_grad_dis}

This section is dedicated to the proof of Proposition~\ref{prop:grad_dis} and Proposition~\ref{prop:p-adjoint-ERK}. In particular, Section~\ref{sec:appendix_proof_grad_dis_leapfrog} and Section~\ref{sec:appendix_proof_grad_dis_2_stage_ERK} deal with the cases of Leapfrog and 2-stage ERK respectively, and Section~\ref{sec:appendix_proof_adjoint_ERK} collects proof for Proposition~\ref{prop:p-adjoint-ERK}.

\subsection{Proof of~\eqref{eqn:gradient}}\label{sec:appendix_proof_grad_dis_leapfrog}

\begin{proof}

To compute the gradient of $E$ on $\Theta$, recall the objective function defined in~\eqref{eqn:def.E} has the constraint~\eqref{eqn:leapfrog}, so we deploy a standard Lagrangian approach in conducting the minimizing process:
\begin{equation}\label{eqn:Lag_LF}
\begin{aligned}
\mathcal{L}(\Theta, \rz; \rp) & := \frac{1}{2} \|g(\rz_L) - y\|_{\mu}^2 - \rp_{\text{init}}^\top (\rz_0 - x) - \rp_0^\top (\rz_1 - \rz_0 - h f(\rz_0, \rtheta_0)) \\
& \ \ \ \ \ \ - \sum_{l=1}^{L-1} \rp_l^\top (\frac{\rz_{l+1}}{2} - \frac{\rz_{l-1}}{2} - h f(\rz_l, \rtheta_l)) , 
\end{aligned}\end{equation}
where $\rp_{\text{init}}$ is the adjoint for the initial data constraint, and $\rp_l$ is the adjoint for the update formula \eqref{eqn:leapfrog}.

On the solution manifold $\rz(\Theta)$ that solves~\eqref{eqn:leapfrog}, the three Lagrange terms are dropped, and
\[
\mathcal{L}(\Theta, \rz(\Theta);\rp) =E(\Theta)\,,
\]
therefore
\begin{equation}\label{eqn:relation_E_L}
\nabla_{\Theta}E = \frac{\partial\mathcal{L}}{\partial\Theta}+ \frac{\partial\mathcal{L}}{\partial\rz}\frac{\partial\rz}{\partial\Theta}\,.
\end{equation}
If there is any $\rp$ such that $\frac{\partial\mathcal{L}}{\partial\rz}=0$, then $\nabla_\Theta E$ can be reduced to $\frac{\partial\mathcal{L}}{\partial\Theta}$, which is further characterized, according to~\eqref{eqn:Lag_LF}:
\begin{equation}
\partial_{\theta_l}E = \frac{\partial \mathcal{L}}{\partial \rtheta_l} = \rp_l^\top \partial_\theta f(\rz_l, \rtheta_l) \, .
\end{equation}

To make $\partial_{\rz}\mathcal{L} = 0$, we compute \begin{subequations}\label{eqn:leapfrog-lagrangian-grad}
\begin{align}
\frac{\partial \mathcal{L}}{\partial \rz_L} & = \revadd{\langle} (g(\rz_L) - y) \revadd{,} \nabla g(\rz_L) \revadd{\rangle_{\mu}}  -  \frac{\rp_{L-1}^\top}{2} \, , \\
\frac{\partial \mathcal{L}}{\partial \rz_{L-1}} & = h \rp_{L-1}^\top \partial_z f(\rz_{L-1}, \rtheta_{L-1}) - \frac{\rp_{L-2}^\top}{2} \, , \\
\frac{\partial \mathcal{L}}{\partial \rz_l} & = \frac{\rp_{l+1}^\top}{2} + h \rp_l^\top \partial_z f(\rz_l, \rtheta_l) - \frac{\rp_{l-1}^\top}{2} \, ,\\
\frac{\partial \mathcal{L}}{\partial \rz_{1}} & = -\rp_0^\top + h \rp_1^\top \partial_z f(\rz_1, \rtheta_1) + \frac{\rp_2^\top}{2}\,. 
\end{align}
\end{subequations}
Setting them to zero, we have \begin{equation}
\rp_{L-1}^\top = 2 \revadd{\langle} (g(\rz_L) - y) \revadd{,} \nabla g(\rz_L) \revadd{\rangle_{\mu}} \, ,
\end{equation}
\begin{equation}
\rp_{L-2}^\top = 2h \rp_{L-1}^\top \partial_z f(\rz_{L-1}, \rtheta_{L-1}) \, ,
\end{equation}
\begin{equation}
\rp_{l-1}^\top = \rp_{l+1}^\top + 2 h \rp_l^\top \partial_z f(\rz_l, \rtheta_l) \, ,
\end{equation} and \begin{equation}
\rp_0^\top = \frac{\rp_2^\top}{2} + h \rp_1^\top \partial_z f(\rz_1, \rtheta_1) \, .
\end{equation}
This completes the proof for the Leapfrog case of the proposition.

\end{proof}

\subsection{Proof of~\eqref{eqn:gradient-2-stage-ERK}}\label{sec:appendix_proof_grad_dis_2_stage_ERK}

\begin{proof}
To compute the gradient of $E$ on $\Theta$, we recall~\eqref{eqn:def.E} is a constrained-optimization. For this type of optimization, a standard technique is to deploy the Lagrange multiplier:
\begin{equation}\label{eqn:def-lagrangian-erk-2stage}
\begin{aligned}
\mathcal{L}(\Theta, \rz;\rp) & := \frac{1}{2} \|g(\rz_L) - y\|_{\mu}^2 - \rp_0^\top (\rz_0 - x) - \sum_{l=0}^{L-1}\rp_{l+\alpha}^\top (\rz_{l+\alpha} - \rz_l - h \alpha f(\rz_{l}, \rtheta_{l})) \\
& \ \ \ \ \ \ - \sum_{l=0}^{L-1}\rp_{l+1}^\top (\rz_{l+1} - \rz_l - h(1 - \frac{1}{2\alpha})f(\rz_l, \rtheta_l) - \frac{h}{2\alpha}f(\rz_{l+\alpha}, \rtheta_{l+\alpha})) \,,
\end{aligned}
\end{equation}
where $\rp_0$ is the Lagrange multiplier for the initial data constraint, and $\rp_{l+\alpha}$ and $\rp_{l}$ are multipliers for equation updates given in~\eqref{eqn:ERK-def}.

We note that on the solution manifold $\rz(\Theta)$ that solves~\eqref{eqn:ERK-def}, the three Lagrange terms are dropped, and
\[
\mathcal{L}(\Theta, \rz(\Theta);\rp) =E(\Theta)\,,
\]
therefore \eqref{eqn:relation_E_L} again holds true.
If we can luckily find a choice of $\rp$ so that $\frac{\partial\mathcal{L}}{\partial\rz}=0$, then $\nabla_\Theta E$ is reduced to $\frac{\partial\mathcal{L}}{\partial\Theta}$, that can be easily computed, as below:
\begin{equation}
    \begin{aligned}
        \partial_{\theta_l}E=\frac{\partial \mathcal{L}}{\partial \rtheta_l} & = h(1 - \frac{1}{2\alpha})\rp_{l+1}^\top \partial_\theta f(\rz_l, \rtheta_l) + h \alpha \rp_{l+\alpha}^\top \partial_\theta f(\rz_l, \rtheta_l) \, 
\\
\partial_{\theta_{l+\alpha}}E=\frac{\partial \mathcal{L}}{\partial \rtheta_{l+\alpha}} & = \frac{h}{2\alpha}\rp_{l+1}^\top \partial_\theta f(\rz_{l+\alpha}, \rtheta_{l+\alpha}) \, .
    \end{aligned}
\end{equation}

To make $\partial_{\rz}\mathcal{L}=0$, we realize:
\begin{subequations}\label{eqn:ERK-lagrangian-2stage-grad}
\begin{align}
\frac{\partial \mathcal{L}}{\partial \rz_L} & = \revadd{\langle} (g(\rz_L) - y) \revadd{,} \nabla g(\rz_L) \revadd{\rangle_{\mu}} - \rp_L^\top \, ,\\
\frac{\partial \mathcal{L}}{\partial \rz_l} & = \rp_{l+1}^\top + h(1 - \frac{1}{2\alpha})\rp_{l+1}^\top \partial_z f(\rz_l, \rtheta_l) - \rp_{l}^\top + \rp_{l+\alpha}^\top + h \alpha \rp_{l+\alpha}^\top \partial_z f(\rz_l, \rtheta_l) \, ,
\\
\frac{\partial \mathcal{L}}{\partial \rz_{l+\alpha}} & = \frac{h}{2\alpha}\rp_{l+1}^\top \partial_z f(\rz_{l+\alpha}, \rtheta_{l+\alpha}) - \rp_{l+\alpha}^\top \,.
\end{align}
\end{subequations}
Setting them to be zero, we have \begin{equation}
\rp_{L}^\top = \revadd{\langle} (g(\rz_L) - y) \revadd{,} \nabla g(\rz_L) \revadd{\rangle_{\mu}} \,,
\end{equation}
and
\begin{equation}\label{eqn:p_l_alpha_erk}
\rp_{l+\alpha}^\top = \frac{h}{2\alpha}\rp_{l+1}^\top \partial_z f(\rz_{l+\alpha}, \rtheta_{l+\alpha}) \ ,
\end{equation} 
and
\begin{equation}\label{eqn:p_l_erk}
\rp_{l}^\top = \rp_{l+1}^\top + h(1 - \frac{1}{2\alpha})\rp_{l+1}^\top \partial_z f(\rz_l, \rtheta_l) + \rp_{l+\alpha}^\top + h \alpha \rp_{l+\alpha}^\top \partial_z f(\rz_l, \rtheta_l) \,.
\end{equation}
Combining~\eqref{eqn:p_l_alpha_erk}-\eqref{eqn:p_l_erk}, we recover~\eqref{eqn:p-adjoint_2-stage-ERK}. Combining~\eqref{eqn:p_l_alpha_erk} and~\eqref{eqn:relation_E_L}, and setting $\partial_{\rz}\mathcal{L}=0$, we obtain~\eqref{eqn:gradient-2-stage-ERK}.

\end{proof}

\subsection{Proof of Proposition~\ref{prop:p-adjoint-ERK}}\label{sec:appendix_proof_adjoint_ERK}

Now we are ready to give the proof of Proposition~\ref{prop:p-adjoint-ERK}.

\begin{proof}
To compute the gradient of $E$ on $\Theta$, we recall that the objective function~\eqref{eqn:def.E} is constrained by~\eqref{eqn:ERK-def}, so the same Lagrangian approach is utilized. Recall the shorthand notation: $\rxi_l = [\rxi_{l,1}^\top \,, \rxi_{l,2}^\top \,, \ldots \,,\ \rxi_{l,\nu}^\top]^\top$ and $f_l$ being $f$ evaluated at $\{\rxi_{l,i}, \theta_{l, c_i}\}_{i=1}^\nu$, we write:
\begin{equation}
\begin{aligned}
\mathcal{L}(\Theta, \rz, \rxi; \rp, \sfq) & := \frac{1}{2} \|g(\rz_L) - y\|_{\mu}^2  - \sum_{l=0}^{L-1}\rp_{l+1}^\top (\rz_{l+1} - \rz_l - h (\mathsf{b}^\top \otimes \mathsf{I}_d) f_l) \\
& \ \ \ \ \ \ - \rp_0^\top (\rz_0 - x) - \sum_{l=0}^{L-1}\sfq_{l}^\top (\rxi_l - (\mathsf{e}_\nu \otimes \rz_l) - h (\mathsf{A} \otimes \mathsf{I}_d) f_l) \,,
\end{aligned}
\end{equation}
where $\rp_0$ is the Lagrange multiplier for the initial data constraint, and $\rp_{l+1}$ and $\sfq_l$ are multipliers for the update formula~\eqref{eqn:ERK-def}.

Similar to the previous sections, to determine the form of the \revdel{cotangent} \revadd{adjoint} vector, we ought to set $\partial_\rz \mathcal{L} = 0$ and $\partial_\rxi \mathcal{L} = 0$. Therefore
\begin{subequations}\label{eqn:ERK-lagrangian-grad}
\begin{align}
\frac{\partial \mathcal{L}}{\partial \rz_L} & = \revadd{\langle} (g(\rz_L) - y) \revadd{,} \nabla g(\rz_L) \revadd{\rangle_{\mu}} - \rp_L^\top \, ,\\
\frac{\partial \mathcal{L}}{\partial \rz_l} & = \rp_{l+1}^\top - \rp_l^\top + \sfq_l^\top (\mathsf{e}_\nu \otimes \mathsf{I}_d) \\
\frac{\partial \mathcal{L}}{\partial \rxi_l} & = h \rp_{l+1}^\top (\mathsf{b}^\top \otimes \mathsf{I}_d) \partial_z f_l + h \sfq_{l}^\top (\mathsf{A} \otimes \mathsf{I}_d) \partial_z f_l - \sfq_l^\top \,.
\end{align}
\end{subequations}
Setting them to zero, we have 
\begin{equation}\label{eqn:p-adjoint-initial-ERK-def}
\rp_L^\top  = \revadd{\langle} (g(\rz_L) - y) \revadd{,} \nabla g(\rz_L) \revadd{\rangle_{\mu}} \, ,
\end{equation}
\begin{equation}\label{eqn:p-adjoint-p-int-def}
\rp_l^\top  = \rp_{l+1}^\top + \sfq_l^\top (\mathsf{e}_\nu \otimes \mathsf{I}_d) \, ,
\end{equation} and \begin{equation}\label{eqn:p-adjoint-q-def}
\sfq_l^\top(\mathsf{I}_{\nu d} - h (\mathsf{A} \otimes \mathsf{I}_d)\partial_z f_l)  = h \rp_{l+1}^\top (\mathsf{b}^\top \otimes \mathsf{I}_d)\partial_z f_l\,.  
\end{equation}
The statement of the proposition holds by combining~\eqref{eqn:p-adjoint-p-int-def} and~\eqref{eqn:p-adjoint-q-def}.
\end{proof}

\section{Proof of Theorem~\ref{theorem:LMMmain}}\label{sec:appendix_proof_Leapfrog}

Here we present the details of the proof of Theorem~\ref{theorem:LMMmain}. To recall the notations, we denote the solution to the ODE~\eqref{eqn:ODE} $z(t)$, and the solution to the adjoint ODE~\eqref{eqn:adjoint} $p(t)$. The ODE-nets using~\eqref{eqn:leapfrog} provides ${\rz}_l$, and the auto-differentiation \revdel{cotangent} \revadd{adjoint} vector $\rp_l$ satisfies~\eqref{eqn:adjoint_dis}.

\subsection{A-priori estimate}\label{sec:a_priori}

A few a-priori estimates are collected here.

\begin{lemma}\label{lem:leapfrog-bound}
$z(t)$, $p(t)$, $\rz_l$, $\rp_l$ and $\hat{\rp}_l$ are all bounded quantities, and the upper bound is independent of $L$. The bound is determined by $C_a, C_b, C_c, C_d$, and the upper bound of $\theta$.
\end{lemma}
This boundedness result comes from an analog to \revdel{~\cite{DiChLiWr:2022overparameterization}}\revadd{~\cite[Lemmas 14 and 24]{DiChLiWr:2022overparameterization}} and we cite them here for the completeness of the paper. The proof extends in a straightforward way and is omitted from here. We emphasize that the boundedness is independent of the density of the discretization.

\begin{lemma}\label{lem:Delta_z}
Let $z(t)$ solve~\eqref{eqn:ODE}, and let $\rz_l$ solve~\eqref{eqn:leapfrog}, then for $l = 0,1,2,\dots,L$:
\begin{equation}
\vert  z(lh) - \rz_l \vert = \mathcal{O}(h^2)\,,
\end{equation}
where the $\mathcal{O}$ notation contains dependence of $C_a, C_b, C_d$, the regularity of $\theta$ up to its second order derivative, and boundedness of $z(t)$.
\end{lemma}

\begin{proof}
This is a standard numerical ODE result. In particular, we denote $\Delta_l^\rz = z(l h) - {\rz}_l$ for all $l$, then recall that $z(0) = x = {\rz}_0$, $\Delta_0^\rz = 0$. Similarly for the first step:
$$\begin{aligned}
z(h) - {\rz}_1 & = z(h) - ({\rz}_0 + h f({\rz}_0, \rtheta_0)) \\
& = (z(0) + h f(z(0), \theta(0)) + \mathcal{O}(h^2)) - (z(0)+ h f(z(0), \theta(0))) \\
& = \mathcal{O}(h^2)\,,
\end{aligned}$$
meaning $\vert \Delta_1^\rz \vert = \mathcal{O}(h^2)$. For $l \geq 1$, we use Taylor expansion and triangle inequality:
$$\begin{aligned}
\vert \Delta_{l+1}^\rz \vert & = \vert z((l+1)h) - {\rz}_{l+1} \vert \\
& = \vert (z((l-1)h) + 2 h f(z(lh), \rtheta_l) + \mathcal{O}(h^3)) - ({\rz}_{l-1} + 2h f({\rz}_l, \rtheta_l)) \vert \\
& \leq \vert z((l-1)h) - {\rz}_{l-1} \vert  + 2h \vert f(z(lh), \rtheta_l) - f({\rz}_l, \rtheta_l)) \vert + \mathcal{O}(h^3) \\
& \leq \vert \Delta_{l-1}^\rz \vert + C h \vert \Delta_l^\rz \vert + \mathcal{O}(h^3)\,.
\end{aligned}$$
where the $\mathcal{O}$ notation include the regularity of the $f$ and $\theta$ up to their second order derivative, and $C$ depends on the first order derivative of $f$. By induction, in each iteration, $\mathcal{O}(h^3)$ error is added to $\Delta^\rz_{l+1}$, and collectively in $L = 1/h$ steps, $|\Delta_l^\rz| = \calO(h^2)$ for $l = 0, 1, 2, \dots, L$.
\end{proof}

Next, we verify the error of the $\hat{\rp}_l$ as an approximator to $\rp_l$.

\begin{lemma}\label{lem:Delta_p}
Let $\rp_l$ solve~\eqref{eqn:adjoint_dis} and let $\hat{\rp}_l$ solve~\eqref{eqn:adjoint_dis} with $\rz_l$ replaced by $z(lh)$, then for all $l \in \{0,\dots, L-1\}$:
\begin{equation}
\vert \hat{\rp}_l - {\rp}_l \vert = \mathcal{O}(h^2)\,.
\end{equation}
where the $\mathcal{O}$ notation includes $C_a, C_b, C_c$, the boundedness of $\theta$, the regularity of $\theta$ up to its second order derivative, and boundedness of $z(t)$ and $p(t)$.
\end{lemma}

\begin{proof}
It is expected due to the stability of the updating formula of $\hat{\rp}_l$. To proceed, we denote $\Delta_l^\rp = \hat{\rp}_l - \rp_l$ for all $l$, then at final time:
$$\begin{aligned}
\vert \Delta_{L-1}^\rp \vert & = \vert \hat{\rp}_{L-1}^\top - {\rp}_{L-1}^\top\vert \\
& = 2 \vert \revadd{\langle} (g(z(1)) - y(x)) \revadd{,} \nabla g(z(1)) \revadd{\rangle_{\mu}} - \revadd{\langle} (g({\rz}_L) - y(x)) \revadd{,} \nabla g({\rz}_L) \revadd{\rangle_{\mu}} \vert \\
& = \mathcal{O}(h^2)\,,
\end{aligned}$$
where we used Lemma~\ref{lem:Delta_z}, the Lipschitz continuity assumption on $g$ and $\nabla g$, and the boundedness of $\rz_L$.
$$\begin{aligned}
\vert \Delta_{L-2}^\rp \vert & = \vert \hat{\rp}_{L-2}^\top - {\rp}_{L-2}^\top \vert \\
& = 2h \vert \hat{\rp}_{L-1}^\top \partial_z f(z((L-1)h), {\rtheta}_{L-1}) - {\rp}_{L-1}^\top \partial_z f({\rz}_{L-1}, {\rtheta}_{L-1}) \vert \\
& \leq 2h \vert \hat{\rp}_{L-1}^\top \vert \cdot \vert \partial_z f(z((L-1)h), {\rtheta}_{L-1}) - \partial_z f({\rz}_{L-1}, {\rtheta}_{L-1}) \vert \\
& \ \ \ \ + 2h\vert \hat{\rp}_{L-1}^\top -{\rp}_{L-1}^\top \vert \cdot \vert \partial_z f(\rz_{L-1}, {\rtheta}_{L-1}) \vert \\
& \leq C_1 h \vert \Delta_{L-1}^\rz \vert + C_2 h \vert \Delta_{L-1}^\rp \vert \\
& = \mathcal{O}(h^2)\,,
\end{aligned}$$
where we have used the triangle inequality, and the regularity of $f$ up to its second order derivative. 
The last equation comes from Lemma~\ref{lem:Delta_z}. For $l=1,2,\dots,L-2$ we call~\eqref{eqn:adjoint_dis} again:
$$\begin{aligned}
\vert \Delta_{l-1}^{\rp} \vert & = \vert  \hat{\rp}_{l-1}^\top - \rp_{l-1}^\top  \vert \\
& = \vert  (\hat{\rp}_{l+1}^\top + 2h \hat{\rp}_{l}^\top \partial_z f(z(lh), {\rtheta}_l)) - (\rp_{l+1}^\top + 2h \rp_{l}^\top \partial_z f(\rz_l, {\rtheta}_l)) \vert \\
& \leq \vert \hat{\rp}_{l+1}^\top - \rp_{l+1}^\top \vert + 2h \vert \hat{\rp}_{l}^\top \vert  \vert \partial_z f(z(lh), {\rtheta}_l) - \partial_z f(\rz_l, {\rtheta}_l) \vert + 2h \vert \hat{\rp}_{l}^\top - \rp_{l}^\top\vert  \vert \partial_z f(\rz_l, {\rtheta}_l) \vert \\
& \leq \vert \Delta_{l+1}^\rp \vert + C_3 h \vert \Delta_{l}^\rp \vert + C_4 h \vert \Delta_{l}^\rz \vert \\
& = \vert \Delta_{l+1}^\rp \vert + C_3 h \vert \Delta_{l}^\rp \vert + \mathcal{O}(h^3)\, ,
\end{aligned}$$
where we have used the triangle inequality, the boundedness of $\hat{\rp}_{l}$ and $\rp_{l}$ and the regularity of $f$ up to its second order derivative. 
Noticing that from $l+1$ and $l$-th step to that $l-1$ we gain an error of $\mathcal{O}(h^3)$. Collectively by iteration using $L = 1/h$, we conclude the proof of the lemma.
\end{proof}

Finally we define the middle agent $\hat{\sfq}$ properly:
\begin{lemma}\label{lem:averaging-on-p}
Let $p(t)$ solve~\eqref{eqn:adjoint} and denote $p(lh) = p(t=lh)$. Define a modification matrix $\tilde{\rT}\in\mathbb{R}^{dL\times dL}$ to be a blockwise diagonal matrix so that $\tilde{\rT} = \{\tilde{\rT}_{i,j}\}_{i,j=0}^{L-1}$ with $\rT_{i,j}$ being its $(i,j)$-th block of size $d\times d$, set to be
\begin{equation}\label{def:T_tilde_matrix}
\tilde{\rT} := \begin{bmatrix}
I & \frac{3}{4}I & 0 & -\frac{1}{4}I & 0 & \cdots \\
\frac{1}{2}I & \frac{1}{2}I & \frac{1}{4}I & 0 & 0 & \cdots \\
0 & \frac{1}{4}I & \frac{1}{2}I & \frac{1}{4}I & 0 & \cdots \\
                                \cdots & \cdots & \ddots & \ddots & \ddots & \cdots \\
\cdots & \cdots & 0 & \frac{1}{4}I & \frac{1}{2}I & \frac{1}{4}I \\
\cdots & \cdots & \cdots & 0 & \frac{1}{4}I & \frac{1}{2}I \\
\end{bmatrix} \ ,
\end{equation}
where $I$ is the $d \times d$ dimensional identity matrix, and define $\hat{\sfq}$ as in~\eqref{eqn:def_hat_q}.
Then
\begin{equation}
\vert p(lh) - \hat{\sfq}_l \vert = \mathcal{O}(h^2)\,,\quad 0\leq l\leq L-1\,,
\end{equation}
where $\mathcal{O}$ depends on $C_a, C_b, C_c, C_d$, and the boundedness of $\theta$ up to its second order derivatives.
\end{lemma}
\begin{remark}
We recall $\rT$ defined in~\eqref{def:T_matrix} can also be formulated into the same form as $\tilde{\rT}$ defined in~\eqref{def:T_tilde_matrix}, with the identity matrix $I$ being of size $n\times n$.
\end{remark}
\begin{proof}
Denote $\Delta_l^\sfq = p(lh) - \hat{\sfq}_l$ for all $l$.
For $\hat{\sfq}_{L-1}$, noticing that
\[
    \hat{\rp}_{L-1}^\top = 2 \revadd{\langle} (g(z(1)) - y(x)) \revadd{,} \nabla g(z(1)) \revadd{\rangle_{\mu}} = 2p^\top(1) \,,
\]
and
\[
    \hat{\rp}_{L-2}^\top = 2h \hat{\rp}_{L-1}^\top \partial_z f(z(1-h), \rtheta_{L-1}) \,,
\]
we have
\allowdisplaybreaks\begin{align*}
\vert \Delta_{L-1}^\sfq \vert & = \vert \hat{\sfq}_{L-1}^\top - p^\top(1-h) \vert \\
& = \left\vert \frac{1}{4}\hat{\rp}_{L-2}^\top + \frac{1}{2}\hat{\rp}_{L-1}^\top - p^\top(1-h)  \right\vert \\
& = \left\vert \frac{h}{2}\hat{\rp}_{L-1}^\top \partial_z f(z(1-h), {\rtheta}_{L-1}) + \frac{1}{2}\hat{\rp}_{L-1}^\top - p^\top(1-h) \right\vert \\
& = \left\vert \frac{h}{2}\hat{\rp}_{L-1}^\top \partial_z f(z(1-h), {\rtheta}_{L-1}) + \left(\frac{1}{2}\hat{\rp}_{L-1}^\top - p^\top(1)\right) - h p^\top(1) \partial_z f(z(1), {\rtheta}_L) \right\vert \\
& \ \ \ \ \ \  + \calO(h^2)\\
& = \frac{h}{2} \vert\hat{\rp}_{L-1}\vert \cdot \vert \partial_z f(z((L-1)h), {\rtheta}_{L-1}) - \partial_z f(z(1), {\rtheta}_L) \vert + \calO(h^2)\\
& =\calO(h^2)\,,
\end{align*}
where we used a Taylor expansion to obtain the fourth equality.
The $\calO$ notation contains dependence on the boundedness of $\hat{\rp}_{L-1}$, the regularity of $\rtheta$ up to its first order derivative, and the regularity of $f$ up to its second order derivative.
To evaluate $\Delta^{\sfq}_{L-2}$, we call the definition of $\hat{\sfq}$, and according to straightforward derivation:
$$\begin{aligned}
\vert \Delta_{L-2}^\sfq \vert & = \vert \hat{\sfq}_{L-2}^\top - p^\top(1-2h) \vert \\
& = \left\vert \frac{1}{4}\hat{\rp}_{L-3}^\top + \frac{1}{2}\hat{\rp}_{L-2}^\top + \frac{1}{4}\hat{\rp}_{L-1}^\top - p^\top(1-2h) \right\vert \\
& = \left\vert \frac{\hat{\rp}_{L-2}^\top}{2}\left(h \partial_z f(z(1-2h), {\rtheta}_{L-2}) + I\right) + \frac{1}{2}\hat{\rp}_{L-1}^\top - p^\top(1-2h) \right\vert \\
& = \big\vert h \hat{\rp}_{L-1}^\top \partial_z f(z(1-h), {\rtheta}_{L-1})(h \partial_z f(z(1-2h), {\rtheta}_{L-2}) + I) \\
& \ \ \ \ + \left(\frac{1}{2}\hat{\rp}_{L-1}^\top - p^\top(1)\right) - 2hp^\top(1) \partial_z f(z(1), {\rtheta}_L) \big\vert + \calO(h^2)\\
& \leq 2h \vert  p(1) \vert \cdot \vert \partial_z f(z(1-h), {\rtheta}_{L-1}) - \partial_z f(z(1), {\rtheta}_L) \vert + \calO(h^2)\\
& = \calO(h^2)\,,
\end{aligned}$$
where the forth equation uses Taylor expansion of $p$ around $t=1$.
The $\calO$ notation contains dependence on the boundedness of $p$ and $z$, the regularity of $\theta$ up to its first order derivative, and the regularity of $f$ up to its second order derivative.
To control $\Delta_l^\sfq$ for $l\leq L-2$, instead of following the definition of $\hat{\sfq}_l$, we notice that $\hat{\sfq}_l$ approximately satisfies the Leapfrog updates for~\eqref{eqn:adjoint}. In particular,
$$\begin{aligned}
& \ \ \ \ \ \ \vert \hat{\sfq}_{l-1}^\top - (\hat{\sfq}_{l+1}^\top + 2h \hat{\sfq}_{l}^\top \partial_z f(z(lh), {\rtheta}_l)) \vert \\
& = \left\vert \frac{\hat{\rp}_{l-2}^\top + 2 \hat{\rp}_{l-1}^\top + \hat{\rp}_{l}^\top}{4} - \frac{\hat{\rp}_{l}^\top + 2 \hat{\rp}_{l+1}^\top + \hat{\rp}_{l+2}^\top}{4} - h\frac{\hat{\rp}_{l-1}^\top + 2 \hat{\rp}_{l}^\top + \hat{\rp}_{l+1}^\top}{2} \partial_z f(z(lh), {\rtheta}_l)) \right\vert \\
& = \left\vert \frac{\hat{\rp}_{l-2}^\top - \hat{\rp}_{l}^\top}{4} + \frac{\hat{\rp}_{l}^\top - \hat{\rp}_{l+2}^\top}{4} + \frac{\hat{\rp}_{l-1}^\top - \hat{\rp}_{l+1}^\top}{2} - h\frac{\hat{\rp}_{l-1}^\top + 2 \hat{\rp}_{l}^\top + \hat{\rp}_{l+1}^\top}{2} \partial_z f(z(lh), {\rtheta}_l)) \right\vert \\
& = \frac{h}{2} \vert \hat{\rp}_{l-1}^\top \partial_z f(z((l-1)h),{\rtheta}_{l-1}) + \hat{\rp}_{l+1}^\top \partial_z f(z((l+1)h),{\rtheta}_{l+1}) \\
& \qquad\qquad\qquad\qquad\qquad\qquad - \hat{\rp}_{l-1}^\top \partial_z f(z(lh),{\rtheta}_{l}) - \hat{\rp}_{l+1}^\top \partial_z f(z(lh),{\rtheta}_{l}) \vert \\
& \leq \frac{h}{2} \vert \hat{\rp}_{l+1} \vert \cdot \vert \partial_z f(z((l-1)h),{\rtheta}_{l-1}) + \partial_z f(z((l+1)h),{\rtheta}_{l+1}) - 2\partial_z f(z(lh),{\rtheta}_{l}) \vert \\
& \ \ \ \ + h^2 \vert \hat{\rp}_{l}  \vert \cdot \vert \partial_z f(z(lh),{\rtheta}_{l}) \cdot (\partial_z f(z((l-1)h),{\rtheta}_{l-1}) - \partial_z f(z(lh),{\rtheta}_{l})) \vert \\
& = \calO(h^3)\,.
\end{aligned}$$
An analog of proof to Lemma~\ref{lem:Delta_z} would show that a Leapfrog method for~\eqref{eqn:adjoint} accumulates $\calO(h^3)$ per iteration, and the inequality above also shows that the updates of $\hat{\sfq}$ accumulates $\calO(h^3)$ per iteration, so together $\Delta^{\sfq}_l$ collects $\calO(h^3)$ when updated from $l$ and $l+1$ to $l$. 
Upon error collection over $L = 1/h$ layers, we have $|\Delta_l^\sfq| = \calO(h^2)$ for all $1\leq l\leq L-2$.
The $\calO$ notation contains dependence on the boundedness of $\hat{\rp}_l$ and $z$, the regularity of $\theta$ up to its second order derivative, and the regularity of $f$ up to its third order derivative.

To finally prove $|\Delta_0^\sfq| = \calO(h^2)$, we apply~\eqref{eqn:def_hat_q} again:
$$\begin{aligned}
& \ \ \ \ \ \ \vert \hat{\sfq}_{0}^\top - (\hat{\sfq}_{2}^\top + 2h \hat{\sfq}_{2}^\top \partial_z f(z(2h), {\rtheta}_2)) \vert \\
& = \left\vert \frac{4\hat{\rp}_{0}^\top + 3 \hat{\rp}_{1}^\top - \hat{\rp}_{3}^\top}{4} - \frac{\hat{\rp}_{1}^\top + 2 \hat{\rp}_{2}^\top + \hat{\rp}_{3}^\top}{4}(I + 2h \partial_z f(z(2h), {\rtheta}_2)) \right\vert \\
& = \left\vert \frac{4\hat{\rp}_{0}^\top + 3 \hat{\rp}_{1}^\top - \hat{\rp}_{3}^\top}{4} - \frac{\hat{\rp}_{1}^\top +  \hat{\rp}_{2}^\top(I - h \partial_z f(z(2h), {\rtheta}_2))}{2}(I + 2h \partial_z f(z(2h), {\rtheta}_2)) \right\vert \\
& = \bigg\vert \frac{4\hat{\rp}_{0}^\top + 3 \hat{\rp}_{1}^\top - \hat{\rp}_{3}^\top}{4} - \frac{1}{2}\hat{\rp}_1^\top (I + 2h \partial_z f(z(2h), {\rtheta}_2)) + \frac{1}{2}\hat{\rp}_2^\top (I -h \partial_z f(z(2h), {\rtheta}_2)) \\
& \qquad\qquad\qquad\qquad - \hat{\rp}_2^\top (I - h \partial_z f(z(2h), {\rtheta}_2))(I + h \partial_z f(z(2h), {\rtheta}_2)) \bigg\vert \\
& = \bigg\vert \frac{4\hat{\rp}_{0}^\top + 3 \hat{\rp}_{1}^\top - \hat{\rp}_{3}^\top}{4} - \frac{1}{2}\hat{\rp}_1^\top (I + 2h \partial_z f(z(2h), {\rtheta}_2)) \\
& \qquad\qquad\qquad\qquad + \frac{1}{2}\hat{\rp}_2^\top (I -h \partial_z f(z(2h), {\rtheta}_2)) - \hat{\rp}_2^\top \bigg\vert +\calO(h^2) \\
& = \vert \frac{4\hat{\rp}_{0}^\top + \hat{\rp}_{1}^\top - \hat{\rp}_{3}^\top}{4} - h \hat{\rp}_1^\top \partial_z f(z(2h), {\rtheta}_2) -\frac{1}{2}\hat{\rp}_2^\top - \frac{1}{2}h \hat{\rp}_2^\top  \partial_z f(z(2h), {\rtheta}_2) \vert +\calO(h^2)\\
& = \left\vert \frac{4\hat{\rp}_{0}^\top + \hat{\rp}_{1}^\top - \hat{\rp}_{3}^\top}{4} - h \hat{\rp}_1^\top \partial_z f(z(h), {\rtheta}_1) -\frac{1}{2}\hat{\rp}_2^\top - \frac{1}{2}h \hat{\rp}_2^\top  \partial_z f(z(2h), {\rtheta}_2) \right\vert +\calO(h^2)\\
& = \left\vert \frac{4\hat{\rp}_{0}^\top + \hat{\rp}_{1}^\top - \hat{\rp}_{3}^\top}{4} - (\hat{\rp}_0^\top - \frac{1}{2}\hat{\rp}_2^\top) -\frac{1}{2}\hat{\rp}_2^\top - \frac{1}{4}(\hat{\rp}_1^\top - \hat{\rp}_3^\top) \right\vert +\calO(h^2)\\
& = \calO(h^2)\,.
\end{aligned}$$
where in the forth last equation we used the boundedness of $\partial_zf$ and the second last equation we used the boundedness of $\partial_{z,\theta}f$ and $\partial_{zz}f$. 
The $\calO$ notation contains dependence on the boundedness of $\hat{\rp}_l$ and $z$, the regularity of $\theta$ up to its second order derivative, and the regularity of $f$ up to its third order derivative.
This completes the whole proof.
\end{proof}

\subsection{Three terms}\label{sec:average}
The a-priori estimates above are now used to control the three terms in~\eqref{eqn:LMMmain-proof-main-estimate}.
\begin{corollary}[Control of Term I in~\eqref{eqn:LMMmain-proof-main-estimate}]\label{cor:term1}
    Let $\nabla_{\Theta}^\rp E$ defined in~\eqref{eqn:def-grad-p-q-leapfrog}, and $\nabla_{\Theta} E$ defined in~\eqref{eqn:gradient}, then $\vert (L \nabla_{\Theta} E)_l -  (L \nabla_{\Theta}^\rp E)_l \vert=\mathcal{O}(h^2)$.
\end{corollary}
\begin{proof}
Using Lemma~\ref{lem:Delta_z} and Lemma~\ref{lem:Delta_p}, we obtain
$$\begin{aligned}
& \ \ \ \ \ \ \vert (L \nabla_{\Theta} E)_l -  (L \nabla_{\Theta}^\rp E)_l \vert \\
& = \vert \rp_l^\top \partial_\theta f(\rz_l, \theta_l) - \hat{\rp}_l^\top \partial_\theta f(z(lh), \rtheta_l) \vert \\ 
& \leq \vert \rp_l \vert \cdot \vert \partial_\theta f(\rz_l, \rtheta_l) - \partial_\theta f(z(lh), \rtheta_l) 
 \vert + \vert \rp_l - \hat{\rp}_l \vert \cdot \vert \partial_\theta f(z(lh), \rtheta_l) \vert \\
 & = \mathcal{O}(h^2)\,.
\end{aligned}$$
Recall the definition of $\rT$ in~\eqref{def:T_matrix}, it is easy to show $\Vert \rT \Vert_\infty = 2$, and this implies \begin{equation}\label{eqn:autodiff-to-p}
\text{Term I} \leq \Vert \rT \Vert_\infty \cdot \vert (L \nabla_{\Theta} E) - (L \nabla_{\Theta}^\rp E) \vert = \mathcal{O}(h^2) \ ,
\end{equation}
where we used the boundedness of $\hat{\rp}_l$ and $\rp_l$ and the regularity of $f$ up to its second order derivative.
\end{proof}

\begin{corollary}[Control of Term III in~\eqref{eqn:LMMmain-proof-main-estimate}]\label{cor:term3}
Let $\nabla_\Theta^\sfq E$ be defined in~\eqref{eqn:def-grad-p-q-leapfrog}, and $\left.\frac{\delta \calE}{\delta {\theta}}\right|_{\theta(t)}(t)$ defined in~\eqref{eqn:gradient_functional}, then:
$\left|\nabla_\Theta^\sfq E-\left.\frac{\delta \calE}{\delta {\theta}}\right|_{\theta(t)}(t)\right|=\mathcal{O}(h^2)$.
\end{corollary}
\begin{proof}
According to the definition
\begin{equation}\label{eqn:q-to-continuous}
\begin{aligned}
 \left|\left(\nabla_\Theta^\sfq E\right)_l-\left.\frac{\delta \calE}{\delta {\theta}}\right|_{\theta(t)}(lh)\right|&= \vert  \hat{\sfq}_l^\top \partial_\theta f(z(lh), \rtheta_l) - p^\top(lh) \partial_\theta f(z(lh), \rtheta_l) \vert \\
& \leq \vert  \hat{\sfq}_l - p(lh) \vert \cdot \vert \partial_\theta f(z(lh), \rtheta_l) \vert \\
& = \mathcal{O}(h^2) \ ,
\end{aligned}
\end{equation}
where we used the boundedness of $z$, the regularity of the first order derivative of $f$ and Lemma~\ref{lem:averaging-on-p}.
\end{proof}

Finally we control the second term.
\begin{lemma}[Control of Term II in~\eqref{eqn:LMMmain-proof-main-estimate}]\label{lem:averaging-on-grad}
Let $\nabla_\Theta^\rp E$ and $\nabla_\Theta^\sfq E$ defined in~\eqref{eqn:def-grad-p-q-leapfrog}, then:
\begin{equation}
\vert L \rT \nabla_{\Theta}^\rp E - L\nabla_{\Theta}^\sfq E \vert = \calO(h^2) \ ,
\end{equation}
where $\mathcal{O}$ depends on $C_a, C_b, C_c, C_d$, boundedness of $\theta$ and its derivative up to the second order.
\end{lemma}

\begin{proof}
For $l \geq 2$, we have, calling the definition of $\rT$ in~\eqref{def:T_matrix}:
$$\begin{aligned}
& \ \ \ \ \ \ \vert (L\rT \nabla_{\Theta}^\rp E)_l - (L \nabla_{\Theta}^\sfq E)_l \vert \\
& = \frac{1}{4} \vert (\hat{\rp}_{l-1}^\top\partial_\theta f(z((l-1)h), {\rtheta}_{l-1}) + 2\hat{\rp}_{l}^\top\partial_\theta f(z(lh), {\rtheta}_{l}) + \hat{\rp}_{l+1}^\top\partial_\theta f(z((l+1)h), {\rtheta}_{l+1})) \\
& \ \ \ \ - (\hat{\rp}_{l-1}^\top + 2\hat{\rp}_{l}^\top + \hat{\rp}_{l+1}^\top) \partial_\theta f(z(lh), {\rtheta}_{l}) \vert \\
& = \frac{1}{4} \vert \hat{\rp}_{l-1}^\top (\partial_\theta f(z((l-1)h), {\rtheta}_{l-1}) - \partial_\theta f(z(lh), {\rtheta}_{l})) \\
& \ \ \ \ + \hat{\rp}_{l+1}^\top (\partial_\theta f(z((l+1)h), {\rtheta}_{l+1}) - \partial_\theta f(z(lh), {\rtheta}_{l})) \vert \\
& = \frac{1}{4} \vert (\hat{\rp}_{l+1}^\top + 2h \hat{\rp}_l^\top \partial_z f(z(lh), {\rtheta}_{l})) (\partial_\theta f(z((l-1)h), {\rtheta}_{l-1}) - \partial_\theta f(z(lh), {\rtheta}_{l})) \\
& \ \ \ \ + \hat{\rp}_{l+1}^\top (\partial_\theta f(z((l+1)h), {\rtheta}_{l+1}) - \partial_\theta f(z(lh), {\rtheta}_{l})) \vert \\
& \leq \frac{1}{4} \vert \hat{\rp}_{l+1}^\top \vert \cdot \vert  \partial_\theta f(z((l+1)h), {\rtheta}_{l+1}) + \partial_\theta f(z((l-1)h), {\rtheta}_{l-1}) - 2 \partial_\theta f(z(lh), {\rtheta}_{l})\vert \\
& \ \ \ \ + \frac{h}{2} \vert \hat{\rp}_l^\top \vert \cdot \vert \partial_\theta f(z((l-1)h), {\rtheta}_{l-1}) - \partial_\theta f(z(lh), {\rtheta}_{l}) \vert \\
& = \calO(h^2)\,,
\end{aligned}$$
where we use the boundedness of $\hat{\rp}_l$ and the regularity of  $f$ up to its third order derivative, and the boundedness of $\theta$ up to its second order derivative.
Similar derivation is obtained for $l = 1$. For $l = 0$:
$$\begin{aligned}
& \ \ \ \ \ \ \vert (L\rT \nabla_{\Theta}^\rp E)_0 - (L \nabla_{\Theta}^\sfq E)_0 \vert \\
& = \frac{1}{4} \vert (4\hat{\rp}_{0}^\top\partial_\theta f(z(0), {\rtheta}_{0}) + 3\hat{\rp}_{1}^\top\partial_\theta f(z(h), {\rtheta}_{1}) - \hat{\rp}_{3}^\top\partial_\theta f(z(3h), {\rtheta}_{3})) \\
& \ \ \ \ - (4\hat{\rp}_{0}^\top + 3\hat{\rp}_{1}^\top - \hat{\rp}_{3}^\top) \partial_\theta f(z(0), {\rtheta}_{0}) \vert \\
& = \frac{1}{4} \vert  3\hat{\rp}_{1}^\top (\partial_\theta f(z(h), {\rtheta}_{1}) - \partial_\theta f(z(0), {\rtheta}_{0})) - \hat{\rp}_{3}^\top (\partial_\theta f(z(3h), {\rtheta}_{3}) - \partial_\theta f(z(0), {\rtheta}_{0})) \\
& = \frac{1}{4} \vert  3 (\hat{\rp}_{3}^\top + 2h\hat{\rp}_{2}^\top \partial_\theta f(z(2h), {\rtheta}_{2})) (\partial_\theta f(z(h), {\rtheta}_{1}) - \partial_\theta f(z(0), {\rtheta}_{0})) \\
& \ \ \ \ - \hat{\rp}_{3}^\top (\partial_\theta f(z(3h), {\rtheta}_{3}) - \partial_\theta f(z(0), {\rtheta}_{0})) \\
& \leq \frac{1}{4} \vert \hat{\rp}_3^\top \vert \cdot \vert -2 \partial_\theta f(z(0), {\rtheta}_{0}) + 3 \partial_\theta f(z(h), {\rtheta}_{1}) - \partial_\theta f(z(3h), {\rtheta}_{3})  \vert \\
& \ \ \ \ + \frac{h}{2} \vert \hat{\rp}_2^\top \vert \cdot \vert \partial_\theta f(z(h), \rtheta_1) - \partial_\theta f(z(0), \rtheta_0)\vert \\
& = \mathcal{O}(h^2).
\end{aligned}$$
This concludes the lemma.
\end{proof}

\section{Proof of Proposition~\ref{prop:autodiff-p-2-stage-ERK} and Theorem~\ref{theorem:ERK-correction}}\label{sec:appendix_2-stage-ERK}
This section is dedicated to 2-stage ERK ODE-nets. We sketch the proof of Proposition~\ref{prop:autodiff-p-2-stage-ERK} and Theorem~\ref{theorem:ERK-correction} in Section~\ref{appendix_proof-autodiff-p-2-stage-ERK}. The proof for Theorem~\ref{theorem:ERK-correction} calls for some a-priori estimates and these technical controls are provided in Section~\ref{appendix_proof_apriori-2-stage-ERK}.

\subsection{Main proof of Proposition~\ref{prop:autodiff-p-2-stage-ERK} and Theorem~\ref{theorem:ERK-correction}}\label{appendix_proof-autodiff-p-2-stage-ERK}

\begin{proof}[Proof of Proposition~\ref{prop:autodiff-p-2-stage-ERK}]
We first note that auto-differentiation is obtained using chain rule~\cite{Linnainmaa1976, Rumelhart1986}, and thus:
\begin{equation}\label{eqn:E_chain_ERK2}
\partial_{\rz_l} E = (\partial_{\rz_{l+1}} E) (\partial_{\rz_l} \rz_{l+1})\,.
\end{equation}
To proceed, we use induction. Firstly, setting $l = L$, then by \eqref{eqn:def.E} and \eqref{eqn:p-adjoint_2-stage-ERK}, $$\partial_{\rz_L} E = \revadd{\langle} (g(\rz_L) - y) \revadd{,} \nabla g(\rz_L) \revadd{\rangle_{\mu}} = \rp_L^\top \ .$$

Now assume that we have $\rp_{l+1}^\top = \partial_{\rz_{l+1}} E$ for some $l$, we are to show this holds for $\rp_{l}^\top = \partial_{\rz_{l}} E$. Applying chain rule  on the update formula $\rz_{l+\alpha} = \rz_l + \alpha h f(\rz_l, \rtheta_l)$, therefore the Jacobian $\partial_{\rz_l} \rz_{l,\alpha}$ is \begin{equation}\label{eqn:half-stage-grad-2-stage-ERK}
\partial_{\rz_l} \rz_{l+\alpha} = I + \alpha h \partial_z f(\rz_l, \rtheta_l) \,.
\end{equation}
Noticing the update formula for integer time steps is $\rz_{l+1} = \rz_l + (1 - \frac{1}{2\alpha})h f(\rz_l, \rtheta_l) + \frac{1}{2\alpha} h f(\rz_{l+\alpha}, \rtheta_{l+\alpha})$, we obtain \begin{equation*}\label{eqn:integer-stage-grad-2-stage-ERK}
\begin{aligned}
&\partial_{\rz_l} \rz_{l+1} \\
 = &I + (1 - \frac{1}{2\alpha}) h \partial_z f(\rz_l, \rtheta_l) + \frac{1}{2\alpha} h \partial_z f(\rz_{l+\alpha}, \rtheta_{l+\alpha}) \partial_{\rz_l} \rz_{l+\alpha} \\
 = &I + (1 - \frac{1}{2\alpha}) h \partial_z f(\rz_l, \rtheta_l) + \frac{1}{2\alpha} h \partial_z f(\rz_{l+\alpha}, \rtheta_{l+\alpha}) + \frac{h^2}{2}  \partial_z f(\rz_{l+\alpha}, \rtheta_{l+\alpha}) \partial_z f(\rz_l, \rtheta_l) \, , 
\end{aligned}
\end{equation*}
where in the first equality we applied chain rule, and in the second equality we plugged in \eqref{eqn:half-stage-grad-2-stage-ERK}. Plugging this in~\eqref{eqn:E_chain_ERK2}:
\[
\begin{aligned}
\partial_{\rz_l} E &= \rp_{l+1}^\top (\partial_{\rz_l} \rz_{l+1}) \\
& = \rp_{l+1}^\top + h \rp_{l+1}^\top \left(\left(1 - \frac{1}{2\alpha}\right) \partial_z f(\rz_l, \rtheta_l) + \frac{1}{2\alpha}\partial_z f(\rz_{l+\alpha}, \rtheta_{l+\alpha}) \right) \\
& \ \ \ \ \ \ + \frac{h^2}{2}\rp_{l+1}^\top \partial_z f(\rz_{l+\alpha}, \rtheta_{l+\alpha})\partial_z f(\rz_l, \rtheta_l) \\
& = \rp_{l}^\top \, ,
\end{aligned}
\]
where we used inductive hypothesis in the first line, and \eqref{eqn:p-adjoint_2-stage-ERK} in the last equality. The proof completes by induction.
\end{proof}

We proceed with the proof of Theorem~\ref{theorem:ERK-correction}.

\begin{proof}[Proof of Theorem~\ref{theorem:ERK-correction}]
According to definitions~\eqref{eqn:2-stage-ERK-mod} and~\eqref{eqn:gradient_functional}, at integer points:
$$\begin{aligned}
& \ \ \ \ |\overline{\partial}_{\rtheta_{l}} E 
- \frac{\delta \calE}{\delta \theta}\bigg\vert_{\theta(t)}(lh) | \\
& = |\rp_{l}^\top \partial_\theta f(\rz_{l}, \rtheta_{l}) - p^\top(lh)\partial_\theta f(z(lh), \rtheta_{l})| \\
& \leq |\rp_{l} - p(lh)| \cdot |\partial_\theta f(\rz_{l}, \rtheta_{l})| \\
& \ \ \ \ \ \ + |p(lh)| \cdot |\partial_\theta f(\rz_{l}, \rtheta_{l}) - \partial_\theta f(z(lh), \rtheta_{l})|\,,
\end{aligned}$$
and at middle-stages:
$$\begin{aligned}
& \ \ \ \ |\overline{\partial}_{\rtheta_{l+\alpha}} E 
- \frac{\delta \calE}{\delta \theta}\bigg\vert_{\theta(t)}((l+\alpha)h) | \\
& = |\sfq_{l+\alpha}^\top \partial_\theta f(\rz_{l+\alpha}, \rtheta_{l+\alpha}) - p^\top((l+\alpha)h)\partial_\theta f(z((l+\alpha)h), \rtheta_{l+\alpha})| \\
& \leq |\sfq_{l+\alpha} - p((l+\alpha)h)| \cdot |\partial_\theta f(\rz_{l+\alpha}, \rtheta_{l+\alpha})| \\
& \ \ \ \ \ \ + |p((l+\alpha)h)| \cdot |\partial_\theta f(\rz_{l+\alpha}, \rtheta_{l+\alpha}) - \partial_\theta f(z((l+\alpha)h), \rtheta_{l+\alpha})|\,,
\end{aligned}$$
where in both we used triangular inequality. If we can provide upper bound of $p$, $z$, $\rz$ and control the difference terms using $\mathcal{O}(h^2)$, the proof is complete. The boundedness of $p$, $z$, $\rz$ are collected in Lemma~\ref{lem:2-stage-ERK-bound}, and the difference terms are controlled in Corollary~\ref{cor:2-stage-ERK-f-theta-error} and~\ref{cor:middle_stage_q} respectively.
\end{proof}

\subsection{A-priori estimates for 2-stage ERK}\label{appendix_proof_apriori-2-stage-ERK}

We record two very useful a-priori results for 2-stage ERK ODE-nets here. 
The first result on boundedness resembles Lemma~\ref{lem:leapfrog-bound}. It is once again an analog to \revdel{~\cite{DiChLiWr:2022overparameterization}}\revadd{~\cite[Lemmas 14 and 24]{DiChLiWr:2022overparameterization}} and its proof is a straightforward extension that has been omitted here.

\begin{lemma}\label{lem:2-stage-ERK-bound}
Fix some $\alpha \in (0,1)$.
Let $\rp_l$ solve~\eqref{eqn:p-adjoint_2-stage-ERK} and let $\hat{\rp}_l$ solve~\eqref{eqn:p-adjoint_2-stage-ERK} with $\rz_l$ and $\rz_{l+\alpha}$ replaced by $z(lh)$ and $z((l+\alpha)h)$ respectively.
Then $z(t)$, $p(t)$, $\rz_l$, $\rz_{l+\alpha}$, $\rp_l$, and $\hat{\rp}_l$ are all bounded quantities, and the upper bound is independent of $L$. The bound is determined by $C_a, C_b, C_c, C_d$, and the upper bound of $\theta$.
\end{lemma}

The next lemma is a standard numerical ODE result.

\begin{lemma}\label{lem:2-stage-ERK-Delta_z}
Fix some $\alpha \in (0,1)$.
Let $z(t)$ solve~\eqref{eqn:ODE} with $z(lh)$ denoting $z(t=lh)$, and let $\rz_l$, $\rz_{l+\alpha}$ be defined as in the 2-stage ERK case in Proposition~\ref{prop:grad_dis}, then for $l = 0,1,2,\dots,L$:
\begin{equation}
\vert  z(lh) - \rz_l \vert = \mathcal{O}(h^2)\,,
\end{equation} and for $l = 0,1,\dots, L-1$: \begin{equation}
\vert  z((l+\alpha)h) - \rz_{l+\alpha} \vert = \mathcal{O}(h^2)\,.
\end{equation}
Here the $\mathcal{O}$ notation contains dependence of $C_a, C_b, C_d$, the regularity of $\theta$ up to its second order derivative, and boundedness of $z(t)$.
\end{lemma}

\begin{proof}
Denote $\Delta_l^\rz = z(l h) - {\rz}_l$ and $\Delta_{l+\alpha}^\rz = z((l+\alpha) h) - {\rz}_{l+\alpha}$ for all $l$, then recall that $z(0) = x = {\rz}_0$, $\Delta_0^\rz = 0$. 
For any $l = 0,1,\dots,L-1$, by Taylor expansion and triangle inequality,
$$\begin{aligned}
|\Delta_{l+\alpha}^\rz| & = |z((l+\alpha) h) - {\rz}_{l+\alpha}| \\
& = |z(lh) + \alpha h f(z(lh), \rtheta_l) + \mathcal{O}(h^2) - \rz_l - \alpha h f(\rz_l, \rtheta_l)| \\
& \leq |z(lh) - \rz_l| + \alpha h|f(z(lh), \rtheta_l) - f(\rz_l, \rtheta_l)|  + \mathcal{O}(h^2) \\
& \leq |z(lh) - \rz_l| + \alpha C h|z(lh) - \rz_l|  + \mathcal{O}(h^2) \\
& = (1 + \alpha C h)|\Delta_{l}^\rz| + \mathcal{O}(h^2) \,,
\end{aligned}$$
$$\begin{aligned}
\vert \Delta_{l+1}^\rz \vert & = \vert z((l+1)h) - {\rz}_{l+1} \vert \\
& = \vert z(lh) + h\left(1 - \frac{1}{2\alpha}\right) f(z(lh), \rtheta_{l})  + \frac{h}{2\alpha} f(z((l+\alpha)h), \rtheta_{l+\alpha}) + \mathcal{O}(h^3)) \\
& \ \ \ \ \ \ - \rz_{l} - h\left(1 - \frac{1}{2\alpha}\right) f(\rz_l, \rtheta_{l})  - \frac{h}{2\alpha} f(\rz_{l+\alpha}, \rtheta_{l+\alpha}) \vert \\
& \leq \vert z(lh) - {\rz}_{l} \vert  + h|1 - \frac{1}{2\alpha}| \cdot \vert f(z(lh), \rtheta_l) - f({\rz}_l, \rtheta_l)) \vert \\
& \ \ \ \ \ \ + \frac{h}{2\alpha} \vert  f(z(l+\alpha)h, \rtheta_{l+\alpha}) - f(\rz_{l+\alpha}, \rtheta_{l+\alpha}) \vert + \mathcal{O}(h^3) \\
& \leq \vert \Delta_{l}^\rz \vert + C h \vert \Delta_{l}^\rz \vert + Ch \vert \Delta_{l+\alpha}^\rz \vert + \mathcal{O}(h^3) \\
& \leq (1 + (C+1)h + \alpha C^2 h^2) \vert \Delta_{l}^\rz \vert + \mathcal{O}(h^3) \,.
\end{aligned}$$
where the $\mathcal{O}$ notation include the regularity of the $f$ and $\theta$ up to their second order derivative, and $C$ depends on the first order derivative of $f$ and the constant $\alpha$. By induction, in each iteration, $\mathcal{O}(h^3)$ error is added to $\Delta^\rz_{l+1}$, and collectively in $L = 1/h$ steps, $|\Delta_l^\rz| = \calO(h^2)$ for $l = 0, 1, 2, \dots, L$.
Plug this back into the estimation for $|\Delta_{l+\alpha}^\rz|$, and we can conclude that $|\Delta_{l+\alpha}^\rz| = \mathcal{O}(h^2)$ for $l = 0,1,\dots, L-1$.
\end{proof}

The lemma above immediately leads to the following corollary:

\begin{corollary}\label{cor:2-stage-ERK-f-theta-error}
Fix some $\alpha \in (0,1)$.
Let $z(t)$ solve~\eqref{eqn:ODE} with $z(lh)$ denoting $z(t=lh)$, and let $\rz_l$, $\rz_{l+\alpha}$ be defined as in the 2-stage ERK case in Proposition~\ref{prop:grad_dis}, then for $l = 0,1,2,\dots,L-1$:
\begin{equation}
\vert  \partial_\theta f(\rz_l, \theta_l) - \partial_\theta f(z(lh), \theta_l) \vert = \mathcal{O}(h^2)\,,
\end{equation} and \begin{equation}
\vert  \partial_\theta f(\rz_{l+\alpha}, \theta_{l+\alpha}) - \partial_\theta f(z((l+\alpha)h), \theta_{l+\alpha}) \vert = \mathcal{O}(h^2)\,.
\end{equation}
\end{corollary}

\begin{proof}
Both estimate follows from Lemma~\ref{lem:2-stage-ERK-Delta_z} and third order continuous differentiability and bounded derivatives assumptions on $f$.
\end{proof}

The error of $\rp_l$ is characterized by Proposition~\ref{prop:2-stage-ERK-p-error}, whose proof is given below.

\begin{proof}[Proof of Proposition~\ref{prop:2-stage-ERK-p-error}]
Let $\hat{\rp}_l$ solve~\eqref{eqn:p-adjoint_2-stage-ERK} with $\rz_l$ and $\rz_{l+\alpha}$ replaced by $z(lh)$ and $z((l+\alpha)h)$ respectively.
For any $l = 0,1,\dots, L$, $$|\rp_l - p(lh)| \leq |\hat{\rp}_l - p(lh)| + |\rp_l - \hat{\rp}_l|  \leq \mathcal{O}(h^2) \, ,$$ where we applied triangular inequality. The two terms are controlled in Lemma~\ref{lem:2-stage-ERK-delta-p-init} and Lemma~\ref{lem:2-stage-ERK-delta-p-hat}.
\end{proof}

The proof of Proposition~\ref{prop:2-stage-ERK-p-error} relies on the following lemmas.

\begin{lemma}\label{lem:2-stage-ERK-delta-p-init}
Let $\rp_l$ solve~\eqref{eqn:p-adjoint_2-stage-ERK} and let $\hat{\rp}_l$ solve~\eqref{eqn:p-adjoint_2-stage-ERK} with $\rz_l$ and $\rz_{l+\alpha}$ replaced by $z(lh)$ and $z((l+\alpha)h)$ respectively, then for all $l \in \{0,1,\dots, L\}$:
\begin{equation}
\vert \hat{\rp}_l - {\rp}_l \vert = \mathcal{O}(h^2)\,,
\end{equation}
where the $\mathcal{O}$ notation includes $C_a, C_b, C_c, C_d$, the boundedness of $\theta$ and its regularity up to its second order derivative, and boundedness of $z(t)$.
\end{lemma}

This result is expected due to the stability of the updating formula of $\hat{\rp}_l$. We now proceed with rigorous proof.
\begin{proof}
Denote $\Delta_l^\rp = \hat{\rp}_l - \rp_l$ for all $l$, then at final time, we have:
$$\begin{aligned}
\vert \Delta_{L}^\rp \vert & = \vert \hat{\rp}_{L}^\top - {\rp}_{L}^\top \vert \\
& = \vert \revadd{\langle} (g(z(1)) - y(x)) \revadd{,} \nabla g(z(1)) \revadd{\rangle_{\mu}} - \revadd{\langle} (g({\rz}_L) - y(x)) \revadd{,} \nabla g({\rz}_L) \revadd{\rangle_{\mu}} \vert \\
& = \mathcal{O}(h^2)\,,
\end{aligned}$$
where we used Lemma~\ref{lem:2-stage-ERK-Delta_z}, the Lipschitz continuity assumption on $g$ and $\nabla g$, and the boundedness of $\rz_L$.
For $l=0,1,\dots,L-1$ we call~\eqref{eqn:p-adjoint_2-stage-ERK} to get
$$\begin{aligned}
\vert \Delta_{l}^{\rp} \vert & = \vert  \hat{\rp}_{l}^\top - \rp_{l}^\top  \vert \\
& = \vert  \hat{\rp}_{l+1}^\top + h \hat{\rp}_{l+1}^\top \left(\left(1 - \frac{1}{2\alpha}\right) \partial_z f(z(lh), \rtheta_l) + \frac{1}{2\alpha}\partial_z f(z((l+\alpha)h), \rtheta_{l+\alpha}) \right) \\
& \ \ + \frac{h^2}{2}\hat{\rp}_{l+1}^\top \partial_z f((z(l+\alpha)h), \rtheta_{l+\alpha})\partial_z f(z(lh), \rtheta_l) - \rp_{l+1}^\top \\
& \ \ - h \rp_{l+1}^\top \left(\left(1 - \frac{1}{2\alpha}\right) \partial_z f(\rz_l, \rtheta_l) + \frac{1}{2\alpha}\partial_z f(\rz_{l+\alpha}, \rtheta_{l+\alpha}) \right) \\
& \ \ - \frac{h^2}{2}\rp_{l+1}^\top \partial_z f(\rz_{l+\alpha}, \rtheta_{l+\alpha})\partial_z f(\rz_l, \rtheta_l) \vert \\
& \leq \vert \hat{\rp}_{l+1} - \rp_{l+1} \vert + h \vert 1 - \frac{1}{2\alpha}\vert  \cdot \vert \hat{\rp}_{l+1} \vert \cdot \vert \partial_z f(z(lh), {\rtheta}_l) - \partial_z f(\rz_l, {\rtheta}_l) \vert \\
& \ \ + h \vert 1 - \frac{1}{2\alpha}\vert  \cdot \vert \hat{\rp}_{l+1} - \rp_{l+1} \vert \cdot \vert \partial_z f(z(lh), {\rtheta}_l) \vert \\
& \ \ + \frac{h}{2\alpha} \vert \hat{\rp}_{l+1} \vert \cdot \vert \partial_z f(z((l+\alpha)h), {\rtheta}_{l+\alpha}) - \partial_z f(\rz_{l+\alpha}, {\rtheta}_{l+\alpha}) \vert \\
& \ \ + \frac{h}{2\alpha} \vert \hat{\rp}_{l+1} - \rp_{l+1} \vert \cdot \vert \partial_z f(z((l+\alpha)h), {\rtheta}_{l+\alpha}) \vert \\
& \ \ + \frac{h^2}{2} \vert \hat{\rp}_{l+1} \vert \cdot \vert \partial_z f(z((l+\alpha)h), {\rtheta}_{l+\alpha}) \vert \cdot \vert \partial_z f(z(lh), {\rtheta}_{l}) - \partial_z f(\rz_{l}, {\rtheta}_{l}) \vert \\
& \ \ + \frac{h^2}{2} \vert \hat{\rp}_{l+1} \vert \cdot \vert \partial_z f(z((l+\alpha)h), {\rtheta}_{l+\alpha}) - \partial_z f(\rz_{l+\alpha}, {\rtheta}_{l+\alpha}) \vert \cdot \vert \partial_z f(\rz_{l}, {\rtheta}_{l}) \vert \\
& \ \ + \frac{h^2}{2} \vert \hat{\rp}_{l+1} - \rp_{l+1} \vert \cdot \vert \partial_z f(\rz_{l+\alpha}, {\rtheta}_{l+\alpha}) \vert \cdot \vert \partial_z f(\rz_{l}, {\rtheta}_{l}) \vert  \\
& \leq  \vert \Delta_{l+1}^\rp \vert + C_1 h \vert \Delta_{l}^\rz \vert + C_2 h \vert \Delta_{l+1}^\rp \vert + C_3 h \vert \Delta_{l+\alpha}^\rz \vert + C_4 h \vert \Delta_{l+1}^\rp \vert + C_5 h^2 \vert \Delta_{l}^\rz \vert \\
& \ \ + C_6 h^2 \vert \Delta_{l+\alpha}^\rz \vert + C_7 h^2 \vert \Delta_{l+1}^\rp \vert \\
& \leq (1 + C_8 h)\vert \Delta_{l+1}^\rp \vert + C_9 h \vert \Delta_{l}^\rz \vert + C_{10} h \vert \Delta_{l+\alpha}^\rz \vert \\
& \leq (1 + C_8 h)\vert \Delta_{l+1}^\rp \vert + C_{11} h^3 \, ,
\end{aligned}$$
where we have used the triangle inequality, the boundedness of $\hat{\rp}_{l}$ and $\rp_{l}$ and the regularity of $f$ up to its second order derivative. 
Noticing that from $l+1$ and $l$-th step to that $l-1$ we gain an error of $\mathcal{O}(h^3)$. Collectively by iteration using $L = 1/h$, we conclude the proof of the lemma.
\end{proof}

\begin{lemma}\label{lem:2-stage-ERK-delta-p-hat}
Let $p(t)$ solve~\eqref{eqn:adjoint}, and let $\hat{\rp}_l$ solve~\eqref{eqn:p-adjoint_2-stage-ERK} with $\rz_l$ and $\rz_{l+\alpha}$ replaced by $z(lh)$ and $z((l+\alpha)h)$ respectively, then for all $l \in \{0,1,\dots, L\}$:
\begin{equation}
\vert \hat{\rp}_l - p(lh) \vert = \mathcal{O}(h^2)\,,
\end{equation}
where the $\mathcal{O}$ notation includes $C_a, C_b, C_c, C_d$, the boundedness of $\theta$ and its regularity up to its second order derivative, and boundedness of $z(t)$ and $p(t)$.
\end{lemma}

\begin{proof}

First by~\eqref{eqn:adjoint} and product rule, \begin{equation}\label{eqn:adjoint-second-diff}
\begin{aligned}
\frac{d^2 p^\top}{dt^2} & = -\frac{d}{dt}(p^\top \partial_z f(z(t), \theta(t))) \\
& = -(\frac{d}{dt}p^\top)\partial_z f(z(t), \theta(t)) - p^\top \frac{d}{dt} (\partial_z f(z(t), \theta(t))) \\
& = p^\top \partial_z f(z(t), \theta(t)) \partial_z f(z(t), \theta(t)) - p^\top \frac{d}{dt} (\partial_z f(z(t), \theta(t))) \, .
\end{aligned}
\end{equation}

So by triangular inequality and Taylor expansion, for any $l = 0,1,\dots, L-1$, \begin{equation}\label{eqn:2-stage-ERK-Taylorapprox}
\begin{aligned}
& \ \ \ \ |p^\top ((l+1)h) + \frac{h}{2} p^\top ((l+1)h) \partial_z f(z((l+1)h), \rtheta_{l+1}) \\
& \ \ \ \ \ \ + \frac{h}{2}p^\top ((l+1)h) \partial_z f(z(lh), \rtheta_{l}) \\
& \ \ \ \ \ \ + \frac{h^2}{2}p^\top ((l+1)h) \partial_z f(z((l+1)h), \rtheta_{l+1})\partial_z f(z((l+1)h), \rtheta_{l+1}) - p^\top(lh)| \\
& = |p^\top ((l+1)h) + h p^\top ((l+1)h) \partial_z f(z((l+1)h), \rtheta_{l+1}) \\
& \ \ \ \ \ \ + \frac{h^2}{2}p^\top ((l+1)h) \partial_z f(z((l+1)h), \rtheta_{l+1})\partial_z f(z((l+1)h), \rtheta_{l+1}) \\
& \ \ \ \ \ \ - \frac{h^2}{2} p^\top ((l+1)h) \frac{\partial_z f(z((l+1)h), \rtheta_{l+1}) - \partial_z f(z(lh), \rtheta_{l})}{h} - p^\top(lh)| \\
& = |p^\top ((l+1)h) - h \frac{d}{dt}p^\top((l+1)h) + \frac{h^2}{2} (\frac{d^2}{dt^2}p^\top((l+1)h) + \mathcal{O}(h)) - p^\top(lh)| \\
& \leq |p^\top ((l+1)h) - h \frac{d}{dt}p^\top((l+1)h) + \frac{h^2}{2} \frac{d^2}{dt^2}p^\top((l+1)h) - p^\top(lh)| + \mathcal{O}(h^3) \\
& = \mathcal{O}(h^3) \, ,
\end{aligned}
\end{equation}
where we used a linear approximation to $\frac{d}{dt}(\partial_z f(z(t), \theta(t)))$ in the second equality, triangular inequality in the first inequality, and second order Taylor approximation,~\eqref{eqn:adjoint} and~\eqref{eqn:adjoint-second-diff} in the last line.

Denote $\Delta_l^{\hat{\rp}} = p(lh) - \hat{\rp}_l$.
Then by our final time condition, $\Delta_L^{\hat{\rp}} = p^\top(1) - \hat{\rp}_L^\top = 0$. For any $l = 0,1,\dots, L-1$, using triangular inequality,
\begin{equation}\label{eqn:2-stage-ERK-delta-p-interim}
\begin{aligned}
&| \Delta_l^{\hat{\rp}} | \\
= & |p^\top (lh) - \hat{\rp}_l^\top | \\
\leq & \mathcal{O}(h^3) + \big|p^\top ((l+1)h) + \frac{h}{2} p^\top ((l+1)h) \partial_z f(z((l+1)h), \rtheta_{l+1}) \\
& \ \ \ \  + \frac{h}{2}p^\top ((l+1)h) \partial_z f(z(lh), \rtheta_{l}) \\
& \ \ \ \  + \frac{h^2}{2}p^\top ((l+1)h) \partial_z f(z((l+1)h), \rtheta_{l+1})\partial_z f(z((l+1)h), \rtheta_{l+1}) - \hat{\rp}_l^\top| \\
= & \mathcal{O}(h^3) + |\left(p^\top ((l+1)h) - \hat{\rp}_{l+1}^\top\right) \\
& \ \ \ \  + \frac{h}{2} \left(p^\top ((l+1)h) - \hat{\rp}_{l+1}^\top \right) \partial_z f(z((l+1)h), \rtheta_{l+1}) \\
& \ \ \ \  + \frac{h}{2}\left(p^\top ((l+1)h) - \hat{\rp}_{l+1}^\top\right) \partial_z f(z(lh), \rtheta_{l}) \\
& \ \ \ \  + \frac{h^2}{2} \left( p^\top ((l+1)h) - \hat{\rp}_{l+1}^\top \right)\partial_z f(z((l+1)h), \rtheta_{l+1})\partial_z f(z((l+1)h), \rtheta_{l+1}) \\
& \ \ \ \  + \hat{\rp}_{l+1}^\top + \frac{h}{2} \hat{\rp}_{l+1}^\top \partial_z f(z((l+1)h), \rtheta_{l+1}) + \frac{h}{2}\hat{\rp}_{l+1}^\top \partial_z f(z(lh), \rtheta_{l}) \\
& \ \ \ \  + \frac{h^2}{2}\hat{\rp}_{l+1}^\top \partial_z f(z((l+1)h), \rtheta_{l+1})\partial_z f(z((l+1)h), \rtheta_{l+1}) - \hat{\rp}_l^\top\big| \\
\leq & \mathcal{O}(h^3) + |\Delta_{l+1}^{\hat{\rp}}| \\
& \ \ \ \ + \frac{h}{2} |\Delta_{l+1}^{\hat{\rp}} | \left(|\partial_z f(z((l+1)h), \rtheta_{l+1})| + |\partial_z f(z(lh), \rtheta_{l})|\right) \\
& \ \ \ \ + \frac{h^2}{2} |\Delta_{l+1}^{\hat{\rp}} | \cdot \big|\partial_z f(z((l+1)h), \rtheta_{l+1})\big| \cdot \big| \partial_z f(z(lh), \rtheta_{l})\big| \\
& \ \ \ \  + \big|\hat{\rp}_{l+1}^\top + \frac{h}{2} \hat{\rp}_{l+1}^\top \partial_z f(z((l+1)h), \rtheta_{l+1}) + \frac{h}{2}\hat{\rp}_{l+1}^\top \partial_z f(z(lh), \rtheta_{l}) \\
& \ \ \ \  + \frac{h^2}{2}\hat{\rp}_{l+1}^\top \partial_z f(z((l+1)h), \rtheta_{l+1})\partial_z f(z((l+1)h), \rtheta_{l+1}) - \hat{\rp}_l^\top\big| \\
\leq & \mathcal{O}(h^3) + (1 + C h)|\Delta_{l+1}^{\hat{\rp}}| \\
& \ \ \ \  + \big|\hat{\rp}_{l+1}^\top + \frac{h}{2} \hat{\rp}_{l+1}^\top \partial_z f(z((l+1)h), \rtheta_{l+1}) + \frac{h}{2}\hat{\rp}_{l+1}^\top \partial_z f(z(lh), \rtheta_{l}) \\
& \ \ \ \  + \frac{h^2}{2}\hat{\rp}_{l+1}^\top \partial_z f(z((l+1)h), \rtheta_{l+1})\partial_z f(z((l+1)h), \rtheta_{l+1}) - \hat{\rp}_l^\top\big|\, ,
\end{aligned}
\end{equation}
where in the first inequality we used triangular inequality and~\eqref{eqn:2-stage-ERK-Taylorapprox}, in the second equality we subtracted and then added $\hat{\rp}_{l+1}^\top$ in order to create the difference $\Delta_{l+1}^{\hat{\rp}}$, in the third inequality we again applied the triangular inequality, and in the last inequality we utilized Lemma~\ref{lem:2-stage-ERK-Delta_z}. 

To bound the last term of~\eqref{eqn:2-stage-ERK-delta-p-interim}, we compute:
$$\begin{aligned}
& \ \ |\hat{\rp}_{l+1}^\top + \frac{h}{2} \hat{\rp}_{l+1}^\top \partial_z f(z((l+1)h), \rtheta_{l+1}) + \frac{h}{2}\hat{\rp}_{l+1}^\top \partial_z f(z(lh), \rtheta_{l}) \\
& \ \ \ \ + \frac{h^2}{2}\hat{\rp}_{l+1}^\top \partial_z f(z((l+1)h), \rtheta_{l+1})\partial_z f(z((l+1)h), \rtheta_{l+1}) - \hat{\rp}_l^\top| \\
& = |\hat{\rp}_{l+1}^\top + \frac{h}{2} \hat{\rp}_{l+1}^\top \partial_z f(z((l+1)h), \rtheta_{l+1}) + \frac{h}{2}\hat{\rp}_{l+1}^\top \partial_z f(z(lh), \rtheta_{l}) \\
& \ \ \ \ + \frac{h^2}{2}\hat{\rp}_{l+1}^\top \partial_z f(z((l+1)h), \rtheta_{l+1})\partial_z f(z((l+1)h), \rtheta_{l+1}) - \hat{\rp}_{l+1}^\top \\
& \ \ \ \ \ \ - h(1-\frac{1}{2\alpha})\hat{\rp}_{l+1}^\top \partial_z f(z(lh),\rtheta_l) - \frac{h}{2\alpha}\partial_z f(z((l+\alpha)h), \rtheta_{l+\alpha}) \\
& \ \ \ \ - \frac{h^2}{2}\hat{\rp}_{l+1}^\top \partial_z f(z((l+\alpha)h), \rtheta_{l+\alpha}) \partial_z f(z(lh), \rtheta_{l})|\\
& \leq \text{Term I}+\text{Term II}
\end{aligned}
$$
where $\text{Term I}$ and $\text{Term II}$ are respectively:
\begin{equation}
\begin{aligned}
\text{Term I}= & \frac{h^2}{2}\left|\hat{\rp}_{l+1}^\top \left(\frac{\partial_z f(z((l+1)h), \rtheta_{l+1})-\partial_z f(z(lh), \rtheta_{l})}{h}\right.\right.\\
&\left.\left. - \frac{\partial_z f(z((l+\alpha)h), \rtheta_{l+\alpha}) - \partial_z f(z(lh), \rtheta_{l})}{\alpha h}\right)\right|
\end{aligned}
\end{equation}
and
\begin{equation}
\begin{aligned}
\text{Term II}=&\frac{h^2}{2}\left|\hat{\rp}_{l+1}^\top \left(\partial_z f(z((l+1)h), \rtheta_{l+1})\partial_z f(z((l+1)h), \rtheta_{l+1})\right.\right.\\
&- \left.\left.\partial_z f(z((l+\alpha)h), \rtheta_{l+\alpha}) \partial_z f(z(lh), \rtheta_{l}) \right)\right|
\end{aligned}
\end{equation}
It is easy to see that the two factors in Term I are both $\frac{d}{dt}\partial_z f(z(lh), \theta(lh)) + \mathcal{O}(h)$, so they cancel out, leaving another $\mathcal{O}(h)$, making $\text{Term I}=\mathcal{O}(h^3)$. To analyze $\text{Term II}$, we also notice that the difference term contribute one extra $\mathcal{O}(h)$, making the whole term $\text{Term II}=\mathcal{O}(h^3)$. Details are heavily repeated and thus not shown, and Lemma~\ref{lem:2-stage-ERK-bound} is called for the boundedness of $\hat{\rp}_{l+1}$ and $z$.

Now plugging these controls back into~\eqref{eqn:2-stage-ERK-delta-p-interim}, we get $$|\Delta_l^{\hat{\rp}}| \leq (1 + Ch)|\Delta_{l+1}^{\hat{\rp}}| + \mathcal{O}(h^3) \ .$$
By induction on $l$, we obtain the desired result.
\end{proof}

The validity of the middle-stage interpolation $\sfq_{l+\alpha}$ defined in~\eqref{eqn:q-def-2-stage-ERK} is due to the following corollary of Proposition~\ref{prop:2-stage-ERK-p-error}.

\begin{corollary}\label{cor:middle_stage_q}
    Let $\sfq_{l+\alpha}$ defined in~\eqref{eqn:q-def-2-stage-ERK} and $p(t)$ solve~\eqref{eqn:adjoint}, then:
    $$|\sfq_{l+\alpha} - p((l+\alpha)h)|=\mathcal{O}(h^2)\,.$$
\end{corollary}
\begin{proof}
By definition:
\begin{equation}\label{eqn:2-stage-ERK-q-mid-bound}
\begin{aligned}
&|\sfq_{l+\alpha} - p((l+\alpha)h)| \\
 \leq & |((1 - \alpha)\rp_l + \alpha \rp_{l+1}) - ((1 - \alpha)p(lh) + \alpha p((l+1)h))| \\
 + & \mathcal{O}(h^2) \\
 \leq & (1 - \alpha) |\rp_l - p(lh)| + \alpha |\rp_{l+1} - p((l+1)h)| + \mathcal{O}(h^2) \\
\leq &\mathcal{O}(h^2) \, ,
\end{aligned}
\end{equation} where we used Taylor expansion in the first inequality, triangular inequality in the second one, and Proposition~\ref{prop:2-stage-ERK-p-error} in the third one.

\end{proof}

\section{Additional Figures}\label{sec:appendix_additional-figures}

This section contains supplementary figures that provide additional insights into the experiments presented in the main text. 

\begin{figure}[ht]
\begin{center}
    \includegraphics[width=0.475\textwidth]{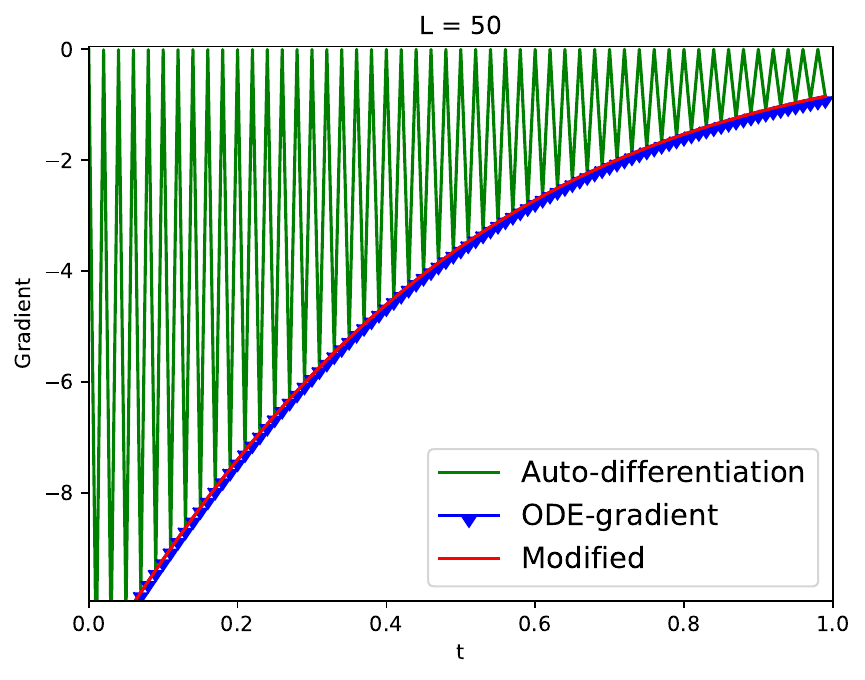}
    \includegraphics[width=0.475\textwidth]{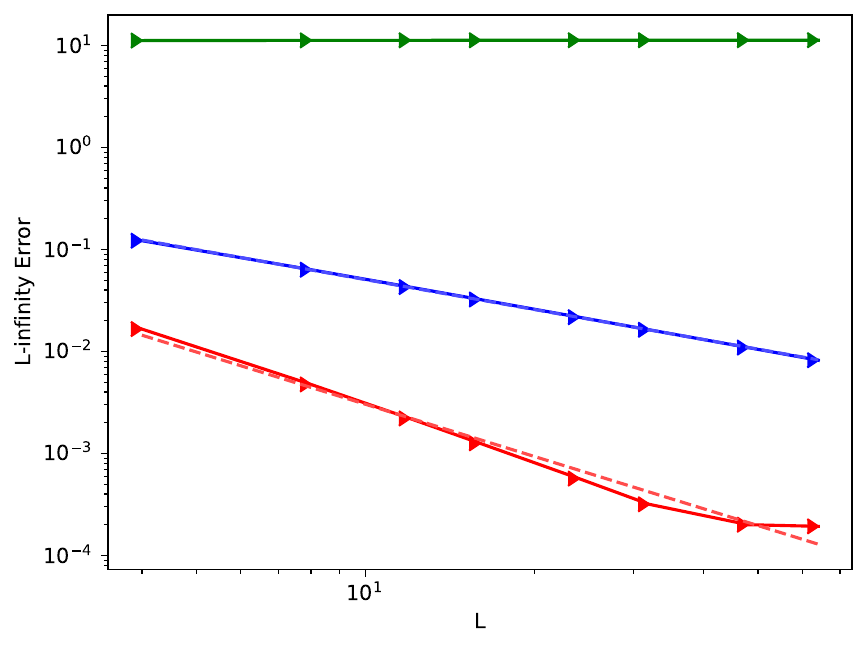}
\caption{Gradient computation for the Midpoint Method.
    The left panel compares $L \nabla_\Theta E$
      (green) with $\NN_\Theta$ formulated in~\eqref{eqn:midpoint}
      with $\frac{\delta\calE}{\delta\theta}$ (blue). $L \nabla_\Theta E$
      oscillates wildly, recovering the true gradient only at half-steps. Our modification (red) yields gradients that match the
      ground-truth. The panel on the right shows convergence with respect to $h$. The $L_\infty$ norm of the
      gradient difference stays as a constant even when $h$ is small, while our modified gradient converges to the ground-truth with a rate of $h^{1.70}$, confirming the theorem. The gradient from using forward Euler obtains the rate of $0.98$ (red and blue).}
\label{fig:failure_auto_midpoint}
\end{center}
\end{figure}

\begin{figure}[ht]
\begin{center}
    \includegraphics[width=0.475\textwidth]{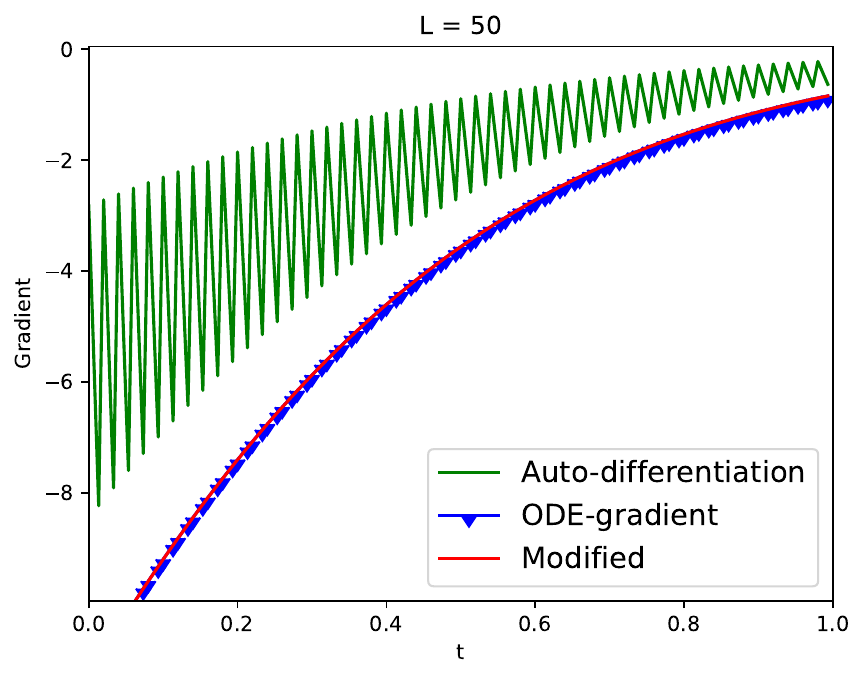}
    \includegraphics[width=0.475\textwidth]{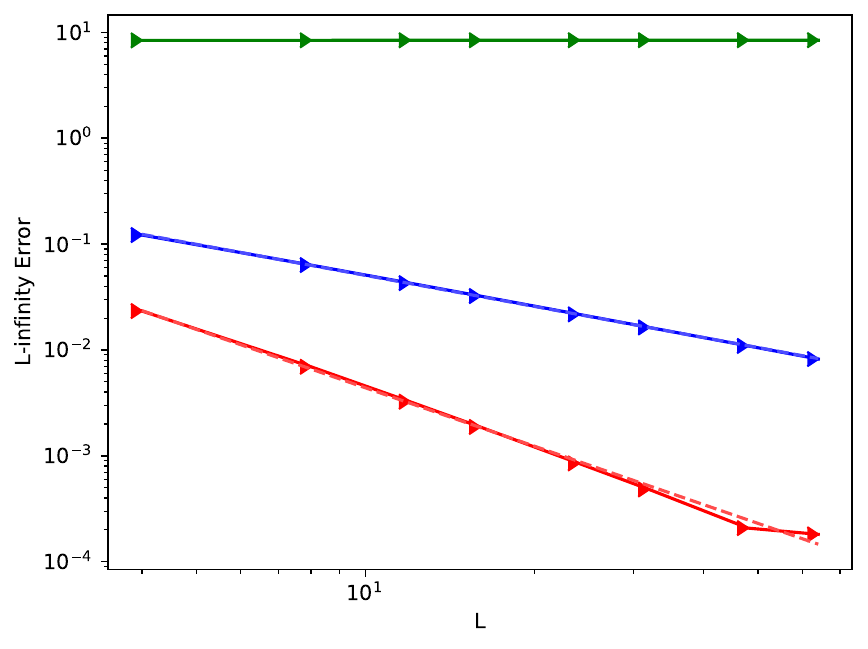}
\caption{Gradient computation for the Ralston Method.
    The left panel compares $L \nabla_\Theta E$
      (green) with $\NN_\Theta$ being defined as the Ralston method
      with $\frac{\delta\calE}{\delta\theta}$ (blue). The output of auto-differentiation $L \nabla_\Theta E$
      oscillates wildly, and is far away from the true gradient. Our modification (red) yields gradients that match the
      ground-truth. The right panel shows convergence. The $L_\infty$ error of the output of auto differentiation stays as a constant. The convergence rates for the modified gradient and
      that computed from forward Euler type discretization are $1.84$
      and $0.98$ respectively (red and blue).}
\label{fig:failure_auto_Ralston}
\end{center}
\end{figure}

\begin{figure}[ht]
\begin{center}
\includegraphics[width=0.75\textwidth]{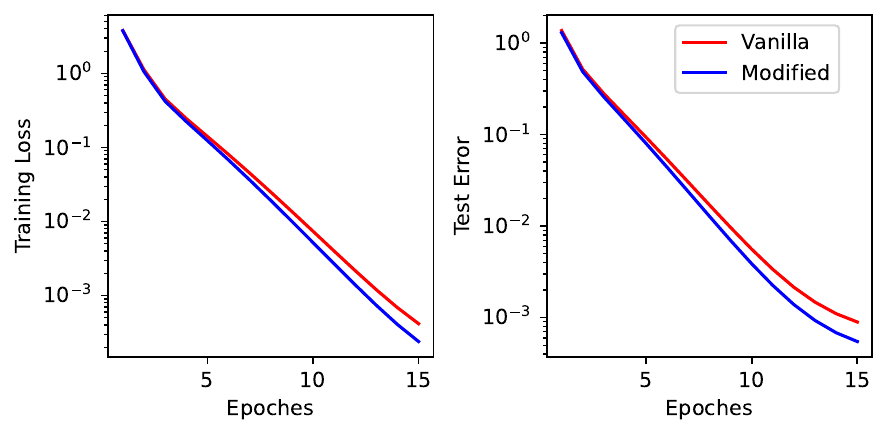} \\
\includegraphics[width=0.75\textwidth]{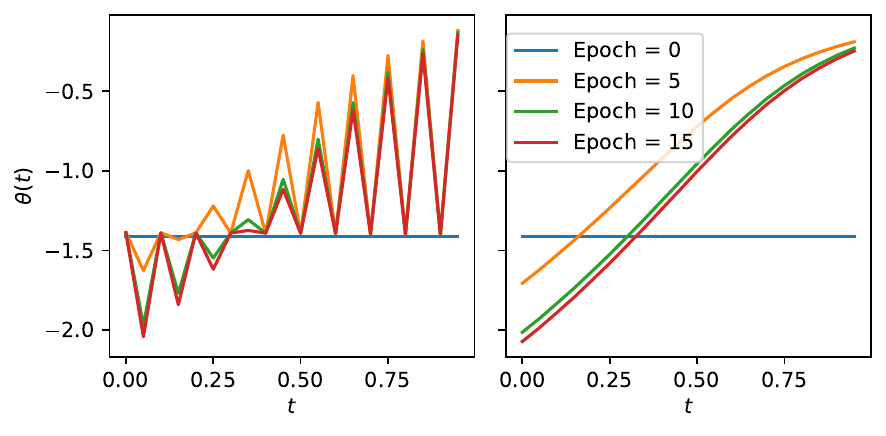} \\
\includegraphics[width=0.6\textwidth]{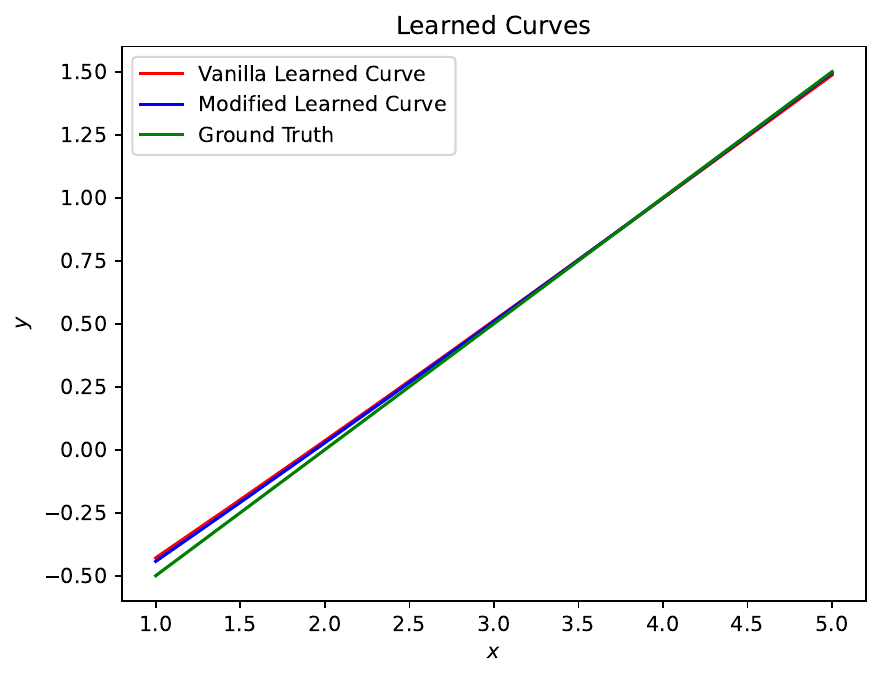}
\caption{
For Midpoint method. 
The top row shows the training loss and test error over epochs. 
The middle row compares $\Theta$ outputs from vanilla \revadd{(unmodified)} auto-differentiation and our modified gradients\revdel{. The}\revadd{, 
and the} bottom plot presents the final reconstruction. 
\revadd{Quantitatively, the corrected gradients achieve much lower final errors:
training loss decreases from $4.2\times10^{-4}$ to $2.4\times10^{-4}$, 
and test error from $9.0\times10^{-4}$ to $5.5\times10^{-4}$. 
These improvements confirm the effectiveness of the proposed correction.}}
\label{fig:Midpoint_Correction}
\end{center}
\end{figure}

\begin{figure}[ht]
\begin{center}
\includegraphics[width=0.75\textwidth]{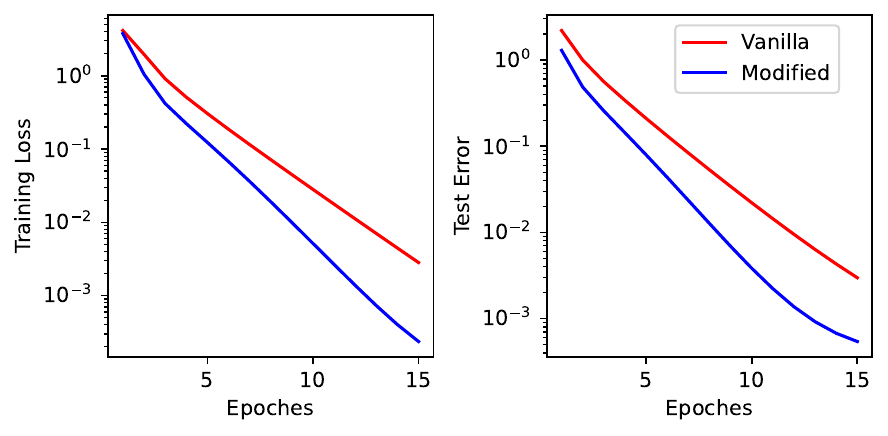} \\
\includegraphics[width=0.75\textwidth]{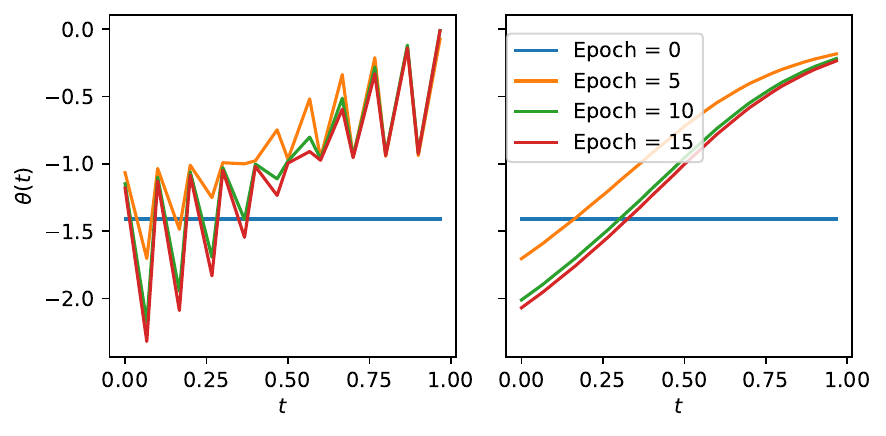} \\
\includegraphics[width=0.6\textwidth]{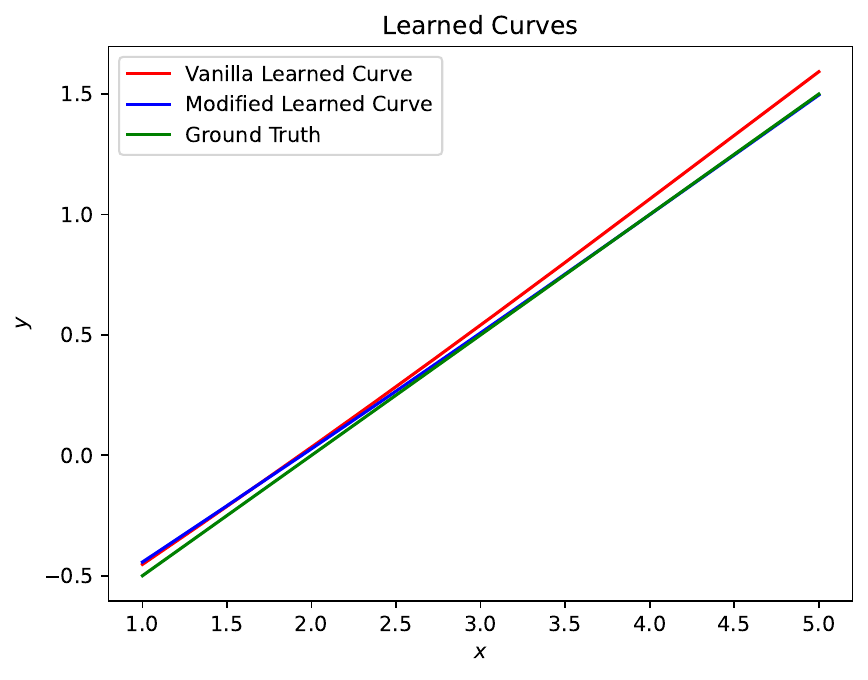}
\caption{For Ralston Method. 
The top row shows the training loss and test error over epochs. 
The middle row compares $\Theta$ outputs from vanilla \revadd{(unmodified)} auto-differentiation and our modified gradients\revdel{. The}\revadd{, 
and the} bottom plot presents the final reconstruction. 
\revadd{Quantitatively, the corrected gradients achieve much lower final errors:
training loss decreases from $2.8\times10^{-3}$ to $2.4\times10^{-4}$, 
and test error from $3.0\times10^{-3}$ to $5.4\times10^{-4}$. 
These improvements confirm the effectiveness of the proposed correction.}}
\label{fig:Ralston_Correction}
\end{center}
\end{figure}

\begin{figure}[ht]
\begin{center}
\includegraphics[width=0.8\textwidth]{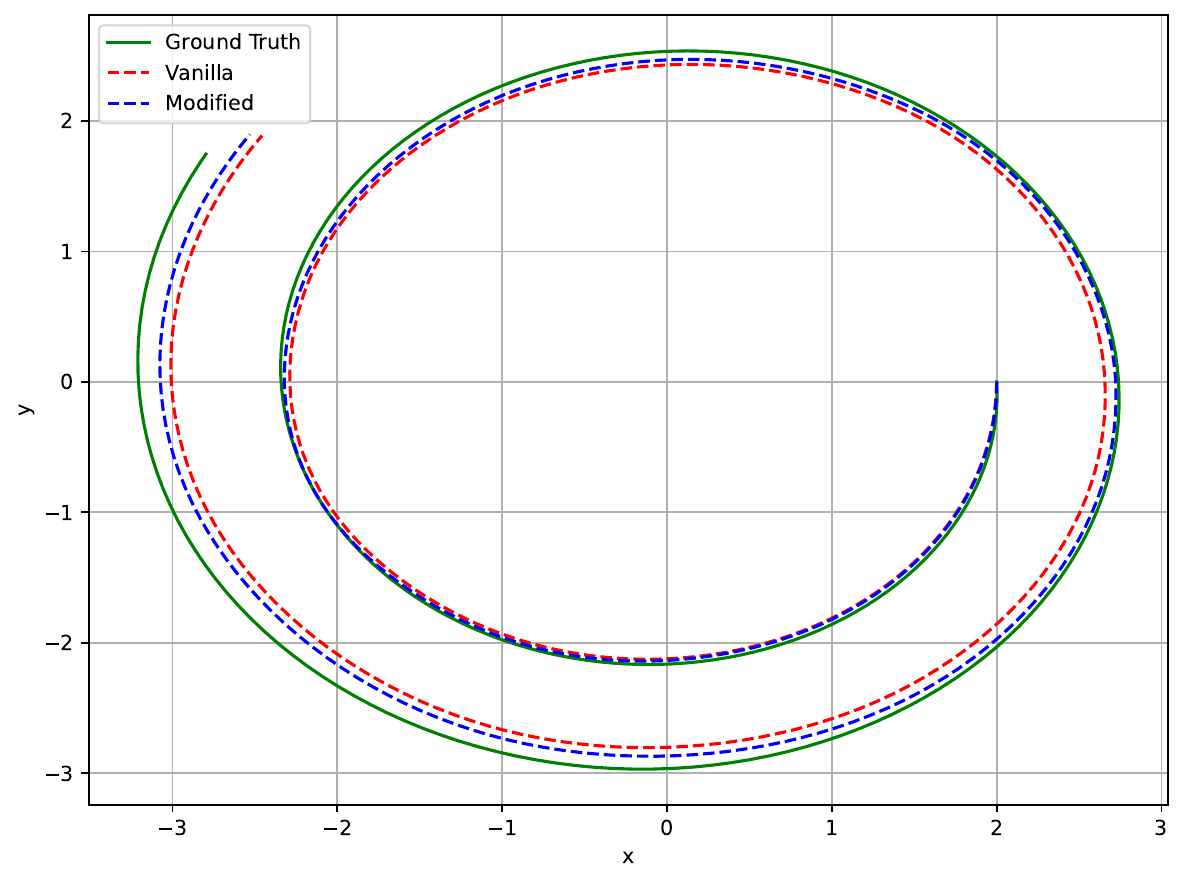}
\caption{For Midpoint Method. The figure compares reconstructions of a spiral learned using auto-differentiation gradients (red) and our modified gradients (blue). Our modification achieves results that align more closely with the ground truth (green).
\revadd{Quantitatively, the trajectory Root Mean Square Error (RMSE) drops from $0.23$ to $0.21$, 
the final-point error from $0.32$ to $0.25$, 
and the maximum deviation from $0.32$ to $0.28$. 
These results confirm that the corrected gradients yield trajectories that 
closely follow the ground-truth dynamics.}}
\label{fig:Midpoint_Spiral}
\end{center}
\end{figure}

\begin{figure}[ht]
\begin{center}
\includegraphics[width=0.8\textwidth]{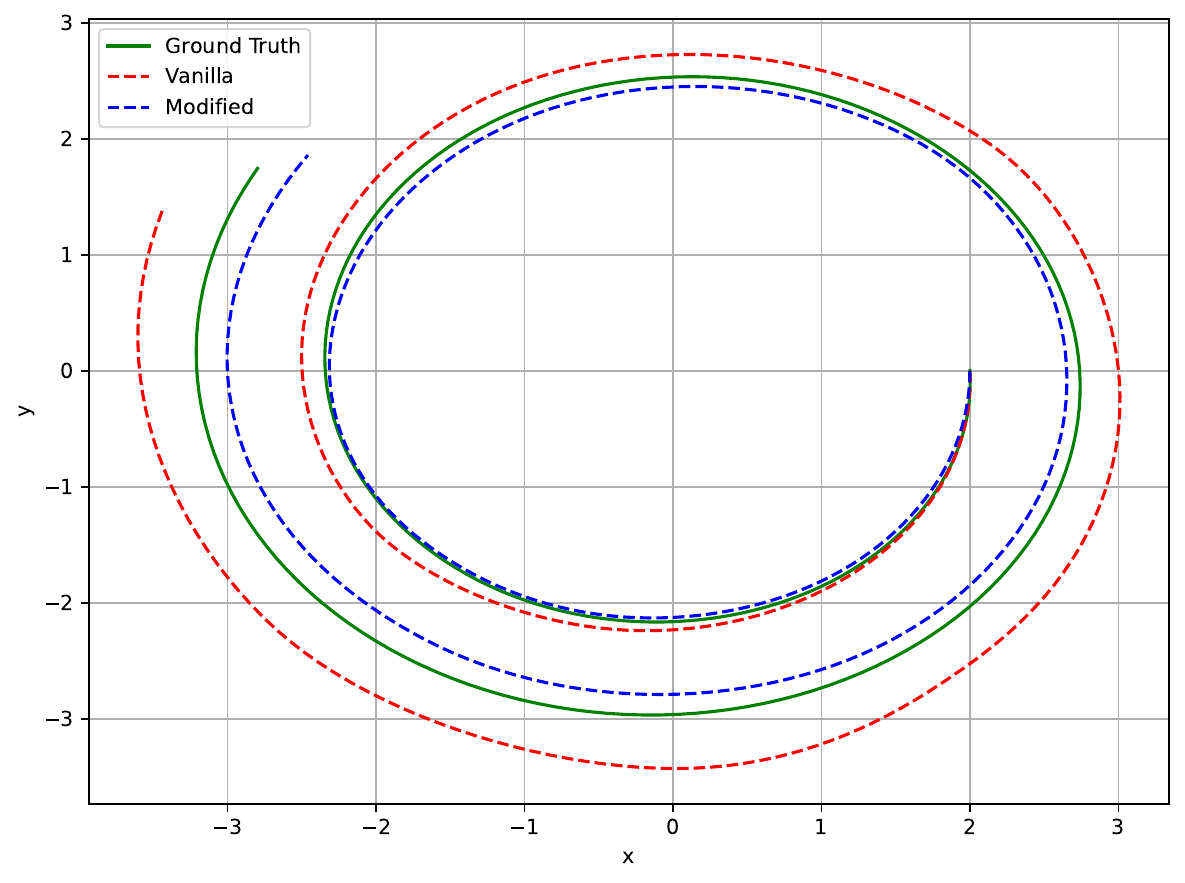}
\caption{For Ralston Method. The figure compares reconstructions of a spiral learned using auto-differentiation gradients (red) and our modified gradients (blue). Our modification achieves results that align more closely with the ground truth (green).
\revadd{Quantitatively, the trajectory Root Mean Square Error (RMSE) drops from $0.40$ to $0.23$, 
the final-point error from $0.79$ to $0.31$, 
and the maximum deviation from $0.79$ to $0.31$. 
These results confirm that the corrected gradients yield trajectories that 
closely follow the ground-truth dynamics.}}
\label{fig:Ralston_Spiral}
\end{center}
\end{figure}

\bibliographystyle{siamplain}
\bibliography{siap_rev_references}

\end{document}